\pdfoutput=1
\documentclass[a4paper, oneside, 11pt]{article}
\usepackage{Styles/personal}
\usepackage[utf8]{inputenc}
\usepackage[english]{babel}
\usepackage{enumitem}
\usepackage{bbm}
\usepackage{mathtools}
\usepackage{booktabs}
\usepackage{multirow}
\usepackage{authblk} 
\usepackage{import}
\newcommand{\bbR}{\ensuremath{\mathbb{R}}}

\newcommand{\calR}{\ensuremath{\mathcal{R}}}

\newcommand{\calS}{\ensuremath{\mathcal{S}}}

\newcommand{\calX}{\ensuremath{\mathcal{X}}}

\newcommand{\bbZ}{\ensuremath{\mathbb{Z}}}

\newcommand{\calF}{\ensuremath{\mathcal{F}}}

\newcommand{\calA}{\ensuremath{\mathcal{A}}}

\newcommand{\calO}{\ensuremath{\mathcal{O}}}

\newcommand{\calC}{\ensuremath{\mathcal{C}}}

\newcommand{\calM}{\ensuremath{\mathcal{M}}}

\newcommand{\frakm}{\ensuremath{\mathfrak{m}}}

\newcommand{\calH}{\ensuremath{\mathcal{H}}}

\newcommand{\calT}{\ensuremath{\mathcal{T}}}

\newcommand{\calK}{\ensuremath{\mathcal{K}}}

\DeclareMathOperator*{\argmin}{argmin}

\newcommand{\E}{\mathrm{E}}

\newcommand{\cov}{\mathrm{Cov}}

\SetCommentSty{mycommfont}
\SetKwInput{KwInput}{Input}         
\SetKwInput{KwOutput}{Output}       

\begin{document}

\title{
\vspace{20pt}\LARGE\textbf{Infinite-dimensional Mahalanobis Distance with Applications to Kernelized Novelty Detection}
}

\author{\vspace{30pt}\LARGE{Nikita Zozoulenko\textsuperscript{1,*}, Thomas Cass\textsuperscript{1,2}, Lukas Gonon\textsuperscript{1,3}}\vspace{20pt}}

\affil{\small\textsuperscript{1} Department of Mathematics, Imperial College London, London, UK\\ \textsuperscript{2} Institute for Advanced Study, Princeton, USA\\ \textsuperscript{3} School of Computer Science, University of St. Gallen, Switzerland\\
\textsuperscript{*} Corresponding author, email: \href{mailto:n.zozoulenko23@imperial.ac.uk}{n.zozoulenko23@imperial.ac.uk}}

\date{\vspace{-20pt}}

\maketitle

\vspace{20pt}
\begin{abstract}%
The Mahalanobis distance is a classical tool used to measure the covariance-adjusted distance between points in $\bbR^d$. In this work, we extend the concept of Mahalanobis distance to separable Banach spaces by reinterpreting it as a Cameron-Martin norm associated with a probability measure. This approach leads to a basis-free, data-driven notion of anomaly distance through the so-called variance norm, which can naturally be estimated using empirical measures of a sample. Our framework generalizes the classical $\bbR^d$, functional $(L^2[0,1])^d$, and kernelized settings; importantly, it incorporates non-injective covariance operators. We prove that the variance norm is invariant under invertible bounded linear transformations of the data, extending previous results which are limited to unitary operators. In the Hilbert space setting, we connect the variance norm to the RKHS of the covariance operator, and establish consistency and convergence results for estimation using empirical measures with Tikhonov regularization. Using the variance norm, we introduce the notion of a kernelized nearest-neighbour Mahalanobis distance, and study some of its finite-sample concentration properties. In an empirical study on 12 real-world data sets, we demonstrate that the kernelized nearest-neighbour Mahalanobis distance outperforms the traditional kernelized Mahalanobis distance for multivariate time series novelty detection, using state-of-the-art time series kernels such as the signature, global alignment, and Volterra reservoir kernels.

\vspace{15pt}
\end{abstract}

\newpage

\section{Introduction}

The Mahalanobis distance \citep{1936Mahalanobis} is a classical tool used to measure the covariance-adjusted distance between points in space on $\bbR^d$. Given a random vector $X$ in $\bbR^d$ with non-singular covariance matrix $\Sigma \in \bbR^{d\times d}$ and mean $\frakm \in \bbR^d$, the Mahalanobis distance of a sample point $y \in \bbR^d$ can be defined in the following three equivalent ways:
\begin{align}
    d_M(y; X) 
    &:= \sqrt{(y-\frakm)^T \Sigma^{-1}(y-\frakm)} \label{eqMahal1}\\
    &= \big\| \Sigma^{-\frac{1}{2}}(y-\frakm)\big\|_{\bbR^d} \nonumber\\
    &= \sqrt{ \sum_{i=1}^d \frac{1}{\lambda_i} \langle y-\frakm, e_i\rangle^2}, \nonumber
\end{align}
where $(e_n, \lambda_n)_{n=1}^N$ are the eigenvector-eigenvalue pairs of the covariance matrix $\Sigma$. Initially proposed by \citet{1936Mahalanobis} for classification, the Mahalanobis distance has since become a cornerstone technique in multivariate analysis \citep{2000MahalanobisDistanceChemometrics}. It is particularly valued for outlier detection, but its utility extends broadly, finding applications in diverse fields such as medicine \citep{2011MedicineAnomalyDetection}, cybersecurity \citep{2020KernelCybersecurity}, chemometrics \citep{2000MahalanobisDistanceChemometrics}, unmanned vehicle detection \citep{2010UnmannedVehicle}, supervised classification \citep{2008MahalanobisLearning}, data clustering \citep{2022MahalanobisClustering}, and financial market anomaly detection \citep{2022CrpytoMarketAnomalyDetection}, to name a few.

In this article, we propose a novel framework for Mahalanobis-type outlier detection on separable Banach and Hilbert spaces, based on a generalized notion of variance norms \citep{2023LyonsShao} and ideas from Cameron-Martin spaces \citep[see e.g.][]{bogachevGaussianMeasures, LifshitsLecturesOnGaussianProcesses, HairerSPDE}. Our extended framework includes the classical Mahalanobis distance on $\bbR^d$ \citep{1936Mahalanobis}, the functional Mahalanobis distance on $(L^2[0,1])^d$ \citep{2015FunctionalMahalanobisTruncated, 2020FunctionalMahalanobis}, and the kernelized Mahalanobis distance \citep{2001KernelizedMahalanobisIEEE} as special cases. Notably, our formulation includes the general case of non-injective covariance operators, which is not addressed in the current literature.

Our work is motivated by the lack of theory surrounding outlier detection on general infinite-dimensional spaces, and more specifically work on novelty detection for time series data using the signature transform \citep{2023LyonsShao, 2022CrpytoMarketAnomalyDetection, 2024PaolaSignatureAnomaly}, an object originating from the theory of rough paths \citep{1998RoughPathTheory}. Existing methods for signature-based outlier detection have been limited to low-dimensional time series due to the exponential $\mathcal{O}(Td^m)$ time complexity in the path dimension $d$ when computing $m$-level truncated signatures of time series of length $T$. Our unified framework addresses this bottleneck, allowing for efficient computations of signature Mahalanobis distances in linear time with respect to $d$ through the use of signature kernels \citep{2019kernelsForSequentiallyOrderedData,2021PdeSignatureKernel}, without truncating the signature, allowing for true infinite-dimensional outlier detection in $\mathcal{O}(T^2 d)$ time. This improvement in time complexity enables these methods to be applied to high-dimensional time series data.

\subsection{Previous Infinite-dimensional Proposals}

The first extension of the finite-dimensional Mahalanobis distance to finite-dimensional, non-linear data was via the kernelized Mahalanobis distance \citep{2001KernelizedMahalanobisIEEE}. It is defined by replacing the implicit dot products in \eqref{eqMahal1} with inner products of a feature map, or equivalently, by positive definite kernel evaluations. This method has been successfully used in applications such as supervised classification \citep{2007MahalanobisDistanceKernelSVM, 2009KernelDiscriminantAnalysis, 2020KernelMahalanobisDistanceTaguchi} and outlier detection \citep{2017KernelMahalanobisFaultDetection, 2018KernelMahalanobisDynamicProcess, 2020KernelCybersecurity}.

Recently, the Mahalanobis distance was generalized to the Hilbert space $L^2[0,1]$ in the context of functional data analysis \citep{2015FunctionalMahalanobisTruncated, 2020FunctionalMahalanobis}. This extension uses Hilbert-Schmidt covariance operators to define functional analogues of the Mahalanobis distance. In this setting, we consider a stochastic process $\big(X(t)\big)_{t \in [0,1]}$ in $L^2[0,1]$ with continuous covariance function $a(s,t) := \cov[X(s),X(t)]$ and functional mean $\frakm(t) := \E[X(t)] \in L^2[0,1]$. The covariance operator $\calK$, defined by $\calK f(t) := \int_0^1 a(s,t) f(s) ds$ for $f\in L^2[0,1]$, is symmetric, positive, compact, and hence diagonalizable by the spectral theorem via the eigenvector-eigenvalue pairs $(e_n, \lambda_n)_{n=1}^\infty$ with non-negative eigenvalues. The naive definition of the functional Mahalanobis distance $d_{FM}$ reads
\begin{align}\label{eqIntroNaive}
    d_{FM}(f; X) := \|\calK^{-\frac{1}{2}} f \|_{L^2[0,1]},
\end{align}
and the difficulty in this infinite-dimensional setting is the non-invertibility of $\calK^\frac{1}{2}$. When the inverse exists, it is given by
$
    \calK^{-\frac{1}{2}} f = \sum_{n=1}^\infty \frac{1}{\sqrt{\lambda_n}}\langle e_n, f\rangle e_n,
$
but since $\calK$ is of trace class, we have that $\sum_{n=1}^\infty \lambda_n < \infty$. This restricts the set of elements for which the inverse is well defined. In fact, if $X$ is a Gaussian process and $f$ is a sample path of $X$, then a classical result from Gaussian probability theory states that $\calK^{-\frac{1}{2}}f$ will almost surely not exist \cite[see e.g.][Theorem 2.4.7]{bogachevGaussianMeasures}. The first paper to use $d_{FM}$ resolved this issue by approximating $\calK$ via its $M$ biggest eigenvalues, where $M$ was determined via cross-validation \citep{2015FunctionalMahalanobisTruncated}. Further theoretical advances were later made to the functional theory under the assumption that $\calK$ is injective \citep{2020FunctionalMahalanobis}, using the RKHS $\calH(\calK) := \calK^{1/2}(L^2[0,1])$ to regularize $d_{FM}$ by considering the minimization problem
    \begin{align}\label{eqIntroRegularization}
        f_\alpha 
        := \argmin_{h\in \calH(\calK)} \|f-h\|^2 + \alpha \|\calK^{-\frac{1}{2}}h\|^2 
        = (\calK+\alpha I)^{-1} \calK f 
        = \sum_{n=1}^\infty \frac{\lambda_n}{\lambda_n+\alpha} \langle f, e_n\rangle e_n,
    \end{align}
for some $\alpha>0$. The regularized functional Mahalanobis distance is defined by replacing $f$ with $f_\alpha$ in \eqref{eqIntroNaive}, or equivalently by considering Tikhonov regularization on the operator $\calK^\frac{1}{2}$. This effectively bypasses the previous invertibility issues, allowing for a well-behaved anomaly distance on $L^2[0,1]$ with theoretical guarantees like consistency of the sample estimator, and well-understood distributional properties under Gaussian assumptions on $X$.

\subsection{Limitations in the current Functional Theory}

In this work we want to address two major limitations in the current theory. The first limitation of the functional Mahalanobis theory is that the sample estimator of $d_{FM}$ is special to the $L^2[0,1]$ setting, and reverts back to finite-dimensional Euclidean theory. The procedure involves discretizing $d$-dimensional sample paths on a grid of $T$ time steps, and computing the Mahalanobis distance in $\mathbb{R}^{Td}$ \citep{2015FunctionalMahalanobisTruncated, FDARamsaySilverman}. While this method works for $L^2[0,1]$, it fails for other inner products that require different infinite-dimensional geometry, such as anomaly detection using the signature transform from rough path theory \citep{2022CrpytoMarketAnomalyDetection, 2023LyonsShao, 2024PaolaSignatureAnomaly, 2024CassSalviLectureNotes}. Furthermore, the theoretical guarantees of \cite{2020FunctionalMahalanobis} were developed for the special case $V = L^2[0,1]$, while we need these properties in the general Hilbert space setting for applications.

The second key limitation we address is the injectivity assumptions in the current functional theory. This is problematic because, when working with sample data, the empirical covariance operator is by definition of finite rank, and therefore non-injective in infinite-dimensional settings. In the functional case, a separate finite-dimensional construction was used for the sample estimator. Our unified framework does not require injectivity, and overcomes these issues by showing that the sample Mahalanobis estimator arises naturally by considering Cameron-Martin spaces with respect to empirical measures. Our framework encompasses both the general infinite-dimensional case and sample estimators within a single theory, eliminating the need for separate constructions.

\subsection{Overview of the Unified Framework and Contributions}

In our unified framework we work on a separable Banach space $(V, \|\cdot\|)$ with continuous dual denoted by $V^*$. Our main object of study is Borel probability measures $\mu$ on $V$ of finite second moment, denoted $\mu \in \calM_V$ as per \cref{defMV}. For such measures $\mu$, the vector-valued mean $\frakm \in V$ and covariance operator $\calK : V^* \to V$
\begin{equation*}
    \frakm := \int_V x d\mu(x), \qquad \qquad \calK f := \int_V (x-\frakm) f(x-\frakm) d\mu(x),
\end{equation*}
are well-defined as Bochner integrals \citep{1987ProbabilityDistributionsOnBanachSpaces}.
The fundamental object we will work with is the $\mu$-variance norm defined by
\begin{equation}\label{eqIntroVarNorm}
    \|x\|_{\mu\text{-cov}} := \sup_{\substack{f \in V^*, \\\cov^\mu[f,f]\leq 1}} f(x),
\end{equation}
which is well-defined for all $x \in V$, but is allowed to be infinite. The set of points for which $\|x\|_{\mu\text{-cov}} < \infty$ is called the Cameron-Martin space of $\mu$, which we denote by $H_\mu$. The measure-theoretic notion \eqref{eqIntroVarNorm} was first suggested in a pre-print of \citet{2023LyonsShao}, but the authors provided no formal theory for the infinite-dimensional case, and the idea was subsequently reworked into a variance norm with respect to a finite sample only, without the use of probability measures or laws. In our extended setting, we define the Banach space Mahalanobis distance as
\begin{equation*}
    d_M(x; \mu) := \|x - \frakm \|_{\mu\text{-cov}},
\end{equation*}
which coincides with the classical $\bbR^d$, functional $(L^2[0,1])^d$, and kernelized Mahalanobis distances, with the added benefit that our definition supports the use of non-injective covariance operators. More importantly, this entails that we no longer need one theory for the covariance operator of a random process, and a different theory for the finite sample estimator. Our framework allows to use the same results, theorems, and definitions for the sample estimator and the underlying random process by considering the variance norm with respect to empirical measures.

An important property of the classical Mahalanobis distance in $\bbR^d$ is its invariance under invertible linear transformations of the data. Whether this remains true in the infinite-dimensional setting has been an open question. \citet{2020FunctionalMahalanobis} was able to prove that invariance holds for unitary operators in the functional case $V = L^2[0,1]$. Using our framework based on variance norms, we fully extend this result in \cref{theoremInvarianceUnderBoundedOperators} to invertible bounded linear operators on Banach spaces.

When specializing to Hilbert spaces, the Cameron-Martin space $H_\mu$ becomes the RKHS of the covariance operator $\calK$, and we are able to express $\|x\|_{\mu\text{-cov}}$ in terms of the eigenvectors and eigenvalues of $\calK$. For applications, we show that the sample estimator of the variance norm is obtained by considering empirical measures of the form $\mu^N = \frac{1}{N}\sum_{i=1}^N \delta_{x_i}$. The empirical $\mu^N$-variance norm can then be computed via the procedure outlined in \cref{theoremEigenvectorsCovOperator}, which is based on an SVD decomposition of the inner product Gram matrix, and is closely related to kernel PCA \citep{1998KernelPCA}. Specifically, this framework allows us to define a kernelized nearest-neighbour Mahalanobis distance, which we show can be computed with the same time complexity as the classical kernelized Mahalanobis distance. This is $\calO(N^3 + N^2K)$ time for fitting the model, and $\calO(N(K+M))$ time for inference, where $N$ is the number of data points, $K$ is the time complexity of a single inner product evaluation, and $M\leq N$ is the number of eigenvalues considered.

A Tikhanov-regularized variance norm can also be obtained in the general Hilbert space setting, similar to the functional $L^2[0,1]$ setting introduced by \citet{2020FunctionalMahalanobis} and in \eqref{eqIntroRegularization}. This allows us to extend the consistency, speed of convergence, and Gaussian distributional results from the functional case to arbitrary separable Hilbert spaces using variance norms. More specifically, we show that the sample estimator based on empirical regularized variance norms converges almost surely to the actual regularized variance norm, with a speed of convergence in probability of $O_P(N^{-\frac{1}{4}})$. Moreover, when $\mu$ is a Gaussian measure, the regularized Mahalanobis distance is equal in distribution to an infinite series of independent standard chi-squared random variables.

We further study the finite-sample properties of the nearest-neighbour distance and its regularized Mahalanobis variant to justify their use in infinite-dimensional settings, where one might expect random points to be almost equidistant. To demonstrate the difficulty of this problem, we show for any set of linearly independent points $\{x_1, ..., x_N\} \subset V$ defining the empirical measure $\mu^N$, that the unregularized empirical Mahalanobis distance satisfies
$
\|x_i - x_j\|_{\mu^N} = 2\sqrt{N}
$
for all $i\neq j$. This highlights the need for a more nuanced analysis and provides additional justification for regularization. To address this, we establish finite-sample concentration bounds for the difference between the nearest- and furthest-neighbour distances under the Hilbert and regularized Mahalanobis norms. Our analysis is based on a Hilbert space Hanson-Wright inequality \citep{2021HansonWrightHilbertSpaces}, and concentration properties of the finite sample covariance operator \citep{2017ConcentrationInequalitiesAndMomentBoundsForSampleCovarianceOperators}. Importantly, we show that the nearest neighbour concentration phenomenon is not governed by the ambient dimension of the space $V$, but rather by the \textit{effective dimensionality} of the covariance operator of the underlying data measure, as given in \cref{defEffectiveDimensionAndRank}.

\subsection{Organization of the Paper}

\cref{sectionCamMartTheory} introduces the covariance operator and the Cameron-Martin space of a probability measure $\mu$ in the Banach space setting. We prove that the variance norm coincides with the classical Cameron-Martin norm, and show that it is invariant under invertible bounded linear transformations of the data. The Mahalanobis and nearest-neighbour Mahalanobis distance is defined, and several important properties are proved.

In \cref{subsectionTheoryHilbert} we specialize to Hilbert spaces, and connect the Cameron-Martin space to the RKHS of the covariance operator $\mu$. We derive computational formulas based on empirical measures with applications to kernel learning. We then define a Tikhanov-regularized variance norm, and derive consistency, speed of convergence, and Gaussian distributional results for the regularized variance norm. 

\cref{sectionNN} studies finite-sample properties of the nearest-neighbour Mahalanobis distance and establishes concentration bounds that justify its use in infinite-dimensional settings.

We conclude the paper with an application to kernelized multivariate time series novelty detection in \cref{sectionApplicationsToTimeSeries}, where we apply our developed framework to various state of the art time series kernels and compare their effectiveness.

\section{Theoretical Foundations of Variance Norms}\label{sectionCamMartTheory}

Throughout this section, we consider a separable Banach space $(V, \|\cdot\|)$ with continuous dual $V^*$, and a Borel probability measure $\mu$ defined on $V$. The primary objective of this section is to develop a comprehensive theory of variance norms on Banach spaces by extending the concepts of Cameron-Martin spaces and norms to non-Gaussian measures. This will allow us to extend the definition of Mahalanobis distance to the Banach space setting.

\subsection{Covariance Operators}\label{subsecCovarianceOperators}

Covariance operators serve as the natural generalization of covariance matrices to infinite-dimensional spaces \citep{1987ProbabilityDistributionsOnBanachSpaces, 2015TheoreticalFoundationsOfFunctionalDataAnalysisWithAnIntroductionToLinearOperators}. These are classical objects in probability theory, and can be defined for random measures --- or equivalently, probably measures --- with finite second moment.

\begin{definition}\label{defMV}
Let $p\geq 1$. A measure $\mu$ on $V$ is said to have finite $p$-th moment if $\|\cdot\| \in L^p( V, \mu)$. We denote by $\calM_V$ the set of all Borel probability measures $\mu$ of finite second moment.
\end{definition}

In particular, empirical measures of the form $\frac{1}{N}\sum_{i=1}^N \delta_{x_i}$ always belong to $\calM_V$, which is essential for computations with observed data. The fundamental object of study in our framework is the covariance operator of $\mu$, which is defined using the classical notion of Bochner integration \citep[see e.g.][]{1991ProbabilityInBanachSpacesLedouxTalagrand}.

\begin{definition}
    Using the Bochner integral, we defined the mean of $\mu \in \calM_V$ as the expectation
    \(
        \frakm = \int_V x d\mu(x) = \E^{x \sim \mu}\big[ x \big].
    \)
    On the continuous dual $V^*$ we define the functional covariance quadratic form $q : V^* \times V^* \to \bbR$ by
    \begin{align*}
    q(f,g) 
    &:= \cov^\mu[ f, g ] 
    = \E^{x\sim\mu}\bigg[ f(x-\frakm)g(x-\frakm) \bigg]
    = \langle f, g \rangle_{L^2(\mu_\frakm)},
    \end{align*}
    for $f,g\in V^*$. Here $\mu_\frakm$ is the measure obtained by shifting $\mu$ by the mean $\frakm$, i.e. the pushforward of $\mu$ under the map $x\mapsto x-\frakm$. The above quantities are well-defined since $\mu$ is assumed to have finite second moment, which allows for the inclusion $V^* \subset L^2(\mu_\frakm)$.
\end{definition}

One observes that $q$ defines a positive quadratic form on $V^*$, but may fail to be an inner product if $q(f,f)=0$ for $f \neq 0$. An alternative characterization of the functional covariance $q$ is through the so-called covariance operator of $\mu$. This is a natural functional-analytic object to study when we no longer have access to the Gaussian tools from the classical theory of Cameron-Martin spaces.

\begin{definition}\label{defCovarianceOperator}
    We define the covariance operator of $\mu \in \calM_V$ to be the bounded linear operator $\calK : V^* \to V$ defined via
    \begin{equation*}
        \calK f := \int_V xf(x) d\mu_\frakm(x),
    \end{equation*}
    for $f \in V^*$.
\end{definition}

The covariance operator of $\mu \in \calM_V$ is well-defined due to the bound $\|xf(x)\| \leq \|f\|_{V^*} \|x\|^2$, which additionally implies that $\calK$ indeed is a bounded linear operator. The following lemma establishes a useful relationship between the covariance operator $\calK$ and the functional covariance $q$, which will be required in the subsequent analysis. These basic properties are well-known, but we include a short proof here for completeness.

\begin{lemma}\label{lemma_qfg_equals_fCg} 
    Let $\mu \in \calM_V$. The covariance operator $\calK : V^* \to V$ is the unique operator satisfying
    \begin{equation*}
        q(f,g) = f( \calK g),
    \end{equation*}
    for all $f,g\in V^*$. Moreover, $\calK$ is compact, and in particular bounded.
\end{lemma}
\begin{proof}
    Suppose that $\calK$ is the covariance operator of $\mu$, and fix $f,g\in V^*$. Using the Bochner integral representation of $\calK$ we obtain that
    \begin{align*}
        q(f,g) 
        &= \int_V f(x)g(x) d\mu_\frakm(x) 
        = \int_V f(xg(x)) d\mu_\frakm(x) 
        = f\bigg(  \int_V xg(x) d\mu_\frakm(x) \bigg)
        = f(\calK g),
    \end{align*}
    where the second to last equality follows from the fact that bounded operators commute with Bochner integrals \citep[see e.g.][Lemma 11.45]{InfiniteDimensionalAnalysisAliprantisBorder}. Conversely, if $q(f,g) = f(\widetilde{\calK} g)$ for all $f,g \in V^*$ and some operator $\widetilde{\calK}$, then $0 = f(\widetilde{\calK} g - \calK g)$. Consequently $\widetilde{\calK} g = \calK g$ by the Hahn-Banach theorem, for each $g\in V^*$.

    As for compactness, suppose that $f_n$ is a bounded sequence in $V^*$, say $\|f_n\|_{V^*} \leq 1$. By Alaoglu's Theorem \cite[Theorem 12.3]{FunctionalAnalysisPeterDLax} there exists a weak*-convergent subsequence $f_{n_k}$ converging to some $f\in V^*$, that is $\lim_k f_{n_k} = f$ pointwise. Since $\|xf_{n_k}(x)\| \leq \|x\|^2$ for all $x\in V$, it follows by the dominated convergence theorem for Bochner integrals \cite[Theorem 11.46]{InfiniteDimensionalAnalysisAliprantisBorder} that
    \begin{align*}
        \lim_k \calK f_{n_k}
        &= \int_V \lim_k  xf_{n_k}(x)d\mu_\frakm(x) = \calK f,
    \end{align*}
    which concludes the proof that $\calK$ is compact.
\end{proof}

\begin{remark}
    An alternative way to define $\calK$ is by using the quadratic form $q$ to first define a linear operator $\calK : V^* \to V^{**}$ via $(\calK f)(g) = q(f,g)$. One then realizes that $\calK f$ actually is an evaluation functional of the vector $\int_V xf(x) d\mu_\frakm(x) \in V$ using the Lemma above, from which the first definition of $\calK$ is recovered.
\end{remark}

\subsection{The Cameron-Martin Space and Extended Covariance Operators}\label{subsecExtendedCamMartSpaces}

A key challenge when working with the covariance operator $\calK: V^* \to V$ is that $\calK$ may be non-injective. We address this by introducing what we term the extended covariance operator, which is injective in the $L^2(\mu_\frakm)$ topology. Our proposed approach of defining Cameron-Martin spaces via this extended covariance operator is to the best of our knowledge novel, and leads to an elegant Gaussian-free approach to variance norms.

\begin{definition} Let $\mu \in \calM_V$. We define the space $\calR_\mu$ to be the closure of $V^*$ in the $L^2(\mu_\frakm)$ topology.
\end{definition}

The space $\calR_\mu$ plays a crucial role throughout this section. As a closed subset of a Hilbert space, $\calR_\mu$ inherits a Hilbert space structure under the $L^2(\mu_\frakm)$ norm. Our goal is to extend $\calK$ to an operator $\calC : \calR_\mu \to V$, where the image $H_\mu = \calC(\calR_\mu)$ will be defined as the Cameron-Martin space of $\mu$.  By an \textit{extension}, we mean that $\mathcal{C}$ coincides with $\mathcal{K}$ on $V^*$. This extension is what enables our subsequent results to apply to empirical measures of a sample, which by definition gives rise to finite-rank, and in particular non-injective, covariance operators. The following proposition shows that by changing topologies from the operator norm on $V^*$ to the $L^2(V, \mu_m)$ topology, we obtain a well-defined extended injective operator. The existence of this extension is not immediately obvious, since the natural estimate $\|f\|_{L^2(\mu_\frakm)}^2 =  \int_V f(x)^2 d\mu_\frakm(x)  \leq \|f\|_{V^*}^2 \int_V \|x\|^2 \mu_\frakm(x)$ goes in the wrong direction.

\begin{proposition}\label{propCalK}
    The covariance operator $\calK : V^* \to V$ extends to a bounded linear operator $\calC : \calR_\mu \to V$, where $\calR_\mu$ is the $L^2(V, \mu_\frakm)$-closure of $V^*$, via the limit
    \begin{align*}
        \calC k &:= \lim_{n\to\infty} \calK f_n = \int_V x k(x) d\mu_\frakm(x),
    \end{align*}
    where $(f_n)_{n=1}^\infty \subset V$ is any sequence converging to $k \in \calR_\mu \subset L^2(V, \mu_\frakm)$. 
\end{proposition}
\begin{proof}
Let $k\in \calR_\mu$. Since $V^*$ is dense in $\calR_\mu$, there exists a sequence $f_n\in V^*$ such that $\|f_n-k\|_{L^2(\mu_\frakm)} \to 0$ as $n\to\infty$. By Hölders inequality we have that
    \begin{align*}
        \bigg\| \calK f_n - \int_V x k(x) d\mu_\frakm(x) \bigg\| 
        &\leq \int_V \|x\| \big|(k-f_n)(x)\big| d\mu_\frakm(x) \\
        &= \bigg(\int_V \|x\|^2 d\mu_\frakm(x)\bigg)^\frac{1}{2}   \bigg(\int_V \big|(k-f_n)(x)\big|^2 d\mu_\frakm(x)\bigg)^\frac{1}{2},
    \end{align*}
which goes to 0 as $n\to\infty$. This holds for any such sequence, and the conclusion follows.
\end{proof}

We can now define the Cameron-Martin space of a general measure $\mu \in \calM_V$. The Cameron-Martin space will be a Hilbert space isometrically isomorphic to $\calR_\mu$, whose norm will naturally be given by the covariance-adjusted distance through the extended covariance operator. This will provide the natural generalization of the Mahalanobis distance for any separable Banach space.

\begin{definition}
    Let $\mu \in \calM_V$ with extended covariance operator  $\calC : \calR_\mu \to V$. We define the Cameron-Martin space $H_\mu$ of $\mu$ to be the set $H_\mu := \calC(R_\mu)$.
\end{definition}

\begin{proposition} 
    The operator $\calC : \calR_\mu \to H_\mu$ is invertible. Hence $H_\mu$ is a Hilbert space under the norm
    \begin{align*}
        \| h\|_{H_\mu} := \| \calC^{-1} h \|_{L^2(\mu_\frakm)}, \qquad
        \langle h, l \rangle_{H_\mu} &:= \big\langle \calC^{-1}h, \calC^{-1}l \big\rangle_{L^2(\mu_\frakm)},
    \end{align*}
    where $h,l \in H_\mu$.
\end{proposition}
\begin{proof}
    We need to prove that $\calC$ is injective. To this end, assume that $\calC k = 0$ for some $k\in \calR_\mu$. By definition there exists a sequence $f_n \in V^*$ such that $f_n \to k$ in $\calR_\mu$. \cref{lemma_qfg_equals_fCg} then implies that
    \begin{align*}
        0 = \lim_n g\big(  \calK f_n \big) = \lim_n \langle g, f_n \rangle_{L^2(\mu_\frakm)} = \langle g, k\rangle_{L^2(\mu_\frakm)},
    \end{align*}
    for all $g\in V^*$. By continuity we obtain that $\langle g, k \rangle_{L^2(\mu_\frakm)} = 0$ for all $g \in \calR_\mu$. Consequently we find that $k=0$ ($\mu_\frakm$-a.e.), which shows that $\calC : \calR_\mu \to H$ is injective. The latter statement of the proposition follows from the Hilbert space structure of $\calR_\mu \subset L^2(V, \mu_\frakm)$ and the linearity of $\calC$.
\end{proof}

The following fundamental result shows that the $\mu$-variance norm $ \| \cdot \|_{\mu\text{-cov}}$ is a genuine norm on a subspace of $V$, and infinite otherwise. More precisely, this subspace is the Cameron-Martin space $H_\mu  \subset V$, and the $\mu$-variance norm coincides with the Cameron-Martin Hilbert norm when restricted to this space. This result provides a solid theoretical foundation for variance-adjusted norms in the general infinite-dimensional setting, bridging the gap between classical infinite-dimensional Gaussian probability theory and the Mahalanobis distance literature. Recall that the $\mu$-variance norm for $x\in V$ is defined as
\[ 
    \|x\|_{\mu\text{-cov}} := \sup_{f \in V^*,\, q(f,f)\leq 1}  f(x),
\]
where $q$ is the functional covariance of $\mu$.

\begin{theorem}\label{TheoremCameronMartinNormDefinitionsCoincide}
    The Cameron-Martin space of $\mu \in \calM_V$ is characterized by
    \begin{align*}
        H_\mu = \{ h\in V : \|h\|_{\mu\text{-cov}} < \infty \}.
    \end{align*}
    Furthermore, the Cameron-Martin norm $\|\cdot\|_{H_\mu}$ and the variance norm $\|\cdot\|_{\mu\text{-cov}}$ coincide on $H_\mu$, or in other words
    \begin{align*}
        \|h\|_{H_\mu} := \| \calC^{-1} h \|_{L^2(\mu_\frakm)} =  \|h\|_{\mu\text{-cov}},
    \end{align*}
    for all $h \in H_\mu$.
\end{theorem}
\begin{proof}
     Suppose that $h = \calC k$ for some $k\in \calR_\mu$. We want to show that the variance norm $\|h\|_{\mu\text{-cov}}$ is finite and equal to $\|h\|_{H_\mu}$. To this end, observe that
     \begin{align*}
        \|h\|_{\mu\text{-cov}}
        &= \sup_{f \in V^*,\, q(f,f)\leq 1} f(\calC k) 
        = \sup_{f \in V^*,\, q(f,f)\leq 1}  \langle f, k \rangle_{L^2(\mu_\frakm)} \\
        &= \sup_{l \in \calR_\mu,\, \|l\|_{L^2(\mu_\frakm)}\leq 1} \langle l, k \rangle_{L^2(\mu_\frakm)}
        = \|\calC^{-1}h\|_{L^2(\mu_\frakm)},
     \end{align*}
     where the third equality follows by the fact that $V^*$ is dense in $\calR_\mu$.

     Conversely, assume that $\|x\|_{\mu\text{-cov}} < \infty$ for some $x \in V$. Let $T_x : V^* \to \bbR$ denote the evaluation functional $T_x g = g(x)$. For $g\in V^*$ with $\|g\|_{L^2(\mu_\frakm)}>0$ we have the bound
     \begin{align*}
         | T_x g | &= |g(x)| 
         \leq \|g\|_{L^2(\mu_\frakm)} \sup_{f \in V^*,\, q(f,f)\leq 1}  f(x),
     \end{align*}
     hence $T_x$ extends to a linear operator $\calT_x : \calR_\mu \to \bbR$ by continuity. More specifically, $\calT_x k$ for $k\in\calR_\mu$ can be defined via $\calT_x k := \lim_n T_x f^{(n)} = \lim_n f^{(n)}(x)$ where $f^{(n)} \in V^*$ is any sequence such that $\|k - f^{(n)}\|_{L^2(\mu_\frakm)} \to 0$. The operator norm for a general bounded operator $\calT \in \calR_\mu^*$ is given by
     \begin{align*}
         \|\calT\|_{\calR^*} 
         &:= \sup_{k \in \calR_\mu,\, \|k\|_{L^2(\mu_\frakm)}\leq 1} \calT k 
         \:\:= \sup_{f \in V^*,\, q(f,f)\leq 1} \calT f,
     \end{align*}
     where the last equality follows by the fact that $V^*$ is dense in $\calR_\mu$. Restricting this to extended evaluation functionals $\calT_x$ we obtain that
     \begin{align*}
         \|\calT_x\|_{\calR^*} 
         &= \sup_{f \in V^*,\, q(f,f)\leq 1} f(x) 
         = \|x\|_{\mu\text{-cov}}.
     \end{align*}
     Since $\calT_x \in \calR_\mu^*$ if and only if the operator norm is finite, we may use the fact that $\calR_\mu$ is a Hilbert space to identify $\calT_x$ with an element of $\calR_\mu$ itself, say $k_x$, such that $\calT_x l = \langle l, k_x\rangle$ for all $l\in \calR_\mu$. If $f\in V^*$, then
    \begin{align*}
        f\big( \calC k_x - x\big) 
        = \langle k_x, f\rangle_{L^2(\mu_\frakm)} - f(x) 
        = f(x) - f(x) = 0,
    \end{align*}
    and it follows by Hahn-Banach that $\calC k_x = h$. This concludes the proof.
\end{proof}
\begin{remark}
    In the above theorem we proved that $h \in H_\mu$ if and only if the evaluation functional $T_h$ extends to a continuous linear functional on $\calR_\mu \subset L^2(V, \mu_\frakm)$. This also proves that $H_\mu$ is a reproducing kernel Hilbert space.
\end{remark}

\begin{remark}
    The literature on Gaussian measures and infinite-dimensional Gaussian probability theory is rich in examples of Cameron-Martin spaces and norms. A classical example is the Wiener measure on $C[0,1]$, the space of continuous functions, where the Cameron-Martin space is the set of all absolutely continuous functions with square integrable derivative, with Cameron-Martin norm
    \(
        \|h\|_{H_\mu} = \int_0^1 |\dot{h}(t)|^2 dt.
    \)
    More generally, there exist expressions for Cameron-Martin spaces and norms for Gaussian measures on $C[0,1]$ in the case where the underlying Gaussian process can be written as an integral with respect to Gaussian white noise. We refer to \cite{LifshitsLecturesOnGaussianProcesses} for further details.
\end{remark}

One important property of the classical Mahalanobis distance on $\bbR^d$ is invariance with respect to non-singular linear transformations of the data. The infinite-dimensional case is more difficult, as \citet{2020FunctionalMahalanobis} noted in the special case $V=L^2[0,1]$ in the Hilbert space setting of functional data analysis. They were able to prove that their functional Mahalanobis distance is invariant with respect to unitary transformations of the data. Using our proposed framework based on variance norms, we are able to extend this result to the Banach space setting for general invertible bounded linear operators. The following proposition comes as a natural consequence of the Cameron-Martin perspective we take in this paper, with a short and elegant proof.

\begin{proposition}\label{theoremInvarianceUnderBoundedOperators}
    The $\mu$-variance norm is invariant under bounded invertible linear transformations of the data. More specifically, if $A:V\to V$ is an invertible bounded linear operator, and $\nu = \mu \circ A^{-1}$, then $\|Ax\|_{\nu\text{-cov}} = \|x\|_{\mu\text{-cov}}$ for all $x\in V$.
\end{proposition}
\begin{proof}
    Denote by $q^\nu$ and $q$ the functional covariance of $\nu$ and $\mu$ respectively. First, we observe by change of variables that
    \begin{align*}
        q^\nu(f, f) = \int f(x- A\frakm)^2 d\nu = \int f(Ax-A\frakm)^2 d\mu = q(f \circ A, f \circ A).
    \end{align*}
    Next, note that $\{g\in V^* : g = f \circ A \} = V^*$, which follows from the fact that the adjoint operator $A^*$ is invertible if and only if $A$ is. Combining the above, we obtain that
    \begin{align*}
        \|Ax\|_{\nu\text{-cov}} = \sup_{f \in V^*,\, q(f\circ A, f\circ A)\leq 1}  f(Ax) = \sup_{g \in V^*,\, q(g, g)\leq 1}  g(x) = \|x\|_{\mu\text{-cov}}.
    \end{align*} \qedhere
\end{proof}

\subsection{Mahalanobis Distance and Conformance Score}\label{subsecMahalConf}

Having introduced the necessary theoretical background in the previous subsections, we are now ready to define the Mahalanobis distance on any separable Banach space $V$. We claim that a natural definition of an anomaly distance on $V$ with respect to a law $\mu$ is the $\mu$-variance norm of a new sample $x$ against the mean $\frakm$, as outlined in the following definition:

\begin{definition}\label{defMahalanobisDistance}
    We define the Mahalanobis distance $d_M(x;\mu)$ of the element $x \in V$ with respect to the measure $\mu \in \calM_V$ to be
    \begin{equation*}
        d_M(x;\mu) := \|x-\frakm\|_{\mu\text{-cov}},
    \end{equation*}
    where $\frakm$ is the mean of $\mu$.
\end{definition}

For real world applications, $\mu$ is often unknown, and an estimator has to be used. This fits naturally within our proposed framework via working with empirical measures: Given a corpus of data $\{x_1, ..., x_N\} \subset V$, the empirical measure of the data is given by $\mu^N = \frac{1}{N}\sum_{i=1}^N \delta_{x_i}$, leading to the sample Mahalanobis distance $d_M(\cdot, \mu^N)$. The following examples demonstrate that our definition coincides with, and in fact extends the Mahalanobis distance in $\bbR^d$ to random variables with possibly degenerate covariance matrices $\Sigma$. Furthermore, the estimator of the Mahalanobis distance will be given by the case where $\mu$ is the empirical measure of the underlying data.

\begin{example}[\textbf{Finite-Dimensional Case}]\label{exampleVariancenormRd}
    Let $\mu$ be a measure on $\mathbb{R}^d$ with covariance matrix $\Sigma = \E^{x \sim \mu}[(x-\frakm)(x-\frakm)^T]$ and mean $\frakm = \E^{x \sim \mu}[x]$. The functional covariance $q$ with respect to $\mu$ is
    \begin{align*}
        q(a, b)
        &= \E^{x \sim \mu}\bigg[ \langle a, x-\frakm\rangle \langle b, x-\frakm\rangle \bigg] 
        =\E^{x \sim \mu} \bigg[ \sum_{i=1}^n\sum_{j=1}^n a_i b_j (x-\frakm)_i(x-\frakm)_j\bigg] \\
        &= \sum_{i=1}^n\sum_{j=1}^n  a_i b_j \Sigma_{i,j}
        = \langle a, \Sigma b \rangle = \langle \Sigma a, b \rangle = a^T \Sigma b,
    \end{align*}
    for all $a,b\in \bbR^d$. \cref{lemma_qfg_equals_fCg} implies that the covariance operator of $\mu$ is simply $\Sigma$, from which \cref{theoremHilbertCameronMartin} gives that the Cameron-Martin space is $H_\mu = \operatorname{Im}(\Sigma)$.  Therefore, by \cref{TheoremCameronMartinNormDefinitionsCoincide}, the $\mu$-variance norm, and thus the Mahalanobis distance, is:
    \begin{align*}
        \|x-\mathfrak{m}\|_{\mu\text{-cov}}^2 
        = \sup_{a \in \bbR^d} \frac{\langle a, x-\frakm\rangle^2}{a^T\Sigma a} 
        = \begin{cases}
            (x-\mathfrak{m})^T \Sigma^{-1} (x-\mathfrak{m}) & \text{if } x-\mathfrak{m} \in \operatorname{Im}(\Sigma), \\
            +\infty & \text{otherwise.}
        \end{cases}
    \end{align*}
    In the specific case of an empirical measure $\mu^N = \frac{1}{N}\sum_{i=1}^N \delta_{x_i}$, $\Sigma$ becomes the sample covariance matrix $\hat{\Sigma}$, and $\mathfrak{m}$ becomes the sample mean $\hat{\mathfrak{m}}$.  This recovers the classical (potentially degenerate) Mahalanobis distance used in finite dimensions \citep{2023LyonsShao}.
\end{example}

While the classical Mahalanobis distance is a widely used metric for outlier detection, measuring distance to the mean may perform poorly in high dimensions. For instance, consider i.i.d.\ standard Gaussian data in $\bbR^d$: the covariance operator in this case is the identity, so the Mahalanobis distance reduces to the Euclidean norm, which concentrates around the sphere of radius $\sqrt{d}$. As a result, the likelihood of a new normal sample being close to the origin is very small if $d$ is large. Consequently, if the Mahalanobis distance is used directly as an anonmaly score, such samples may incorrectly be classified as being outliers. An alternative approach in such scenarios is to use the $k$-nearest-neighbour distance to the normal corpus \citep{2004OutlierDetectionKNN, 2011MahalanobisDistanceKNNSemiconductor, 2023LyonsShao}. This approach requires choosing a metric for calculating the nearest-neighbours. In the finite-dimensional setting the Euclidean, Minkowski, Manhattan, or even the Mahalanobis distance itself are commonly used. \citet{2023LyonsShao} coined the term \textit{conformance score} for the case when the Mahalanobis distance is used in conjunction with the $1$-nearest-neighbour Mahalanobis distance \citep[see also][]{2011MahalanobisDistanceKNNSemiconductor}. Below, we generalize this notion to the Banach space setting for laws $\mu \in \calM_V$. Note that the notion of variance norm by \citet{2023LyonsShao} is restricted to empirical measures only, while our unified framework considers any law $\mu$. In \cref{sectionComputingVarNorms,sectionApplicationsToTimeSeries}, we derive computational formulas for the infinite-dimensional conformance score and evaluate these anomaly metrics in the context of time series novelty detection.

\begin{definition}\label{defConformanceScore}
    Let $\mu \in \calM_V$ for a Banach space $V$, and let $\{x_1, \cdots, x_N\} \subset V$ be a corpus of observed data. We define the conformance score $d_C(x; \mu)$ of $x$ with respect to $\mu$ and the corpus as
    \begin{equation*}
        d_C(x; \mu) := \min_{1\leq i \leq N}\|x-x_i\|_{\mu\text{-cov}}.
    \end{equation*}
\end{definition}

The following proposition extends \cref{exampleVariancenormRd} to the case of empirical measures on a Banach space. In this setting, the Cameron-Martin space $H_{\mu^N}$ associated with the empirical measure $\mu^N$ is finite-dimensional. Nevertheless, a challenge lies in the fact the variance norm depends non-trivially on all of $V^*$, which is infinite-dimensional. The result also establishes the basis-independence of the variance norm with respect to the basis in which the data is observed. This extends the results of \citet{2023LyonsShao} from the finite-dimensional setting to the Banach space setting.

\begin{proposition}\label{theoremEmpiricalBanachVarianceNorm}
    Let $\mu^N = \frac{1}{N}\sum_{i=1}^N \delta_{x_i}$ be an empirical measure. Write $y_i := x_i-\widehat{\frakm}$, $i \in \{1, \cdots, N\}$ for the centered data, where $\widehat{\frakm} = \frac{1}{N}\sum_{i=1}^N x_i$ is the empirical mean. Then the following statements hold:
    \begin{enumerate}[label=(\roman*)]
        \item The Cameron-Martin space is $H_{\mu^N} = \operatorname{span}\{y_1, \cdots, y_N\}$.
        \item Let $\{e_1, \cdots, e_M\}$ be a basis of $H_{\mu^N}$, and $A: V \to H_{\mu^N}$ be a surjective projection. Denote by $a^{(x)} \in \bbR^M$ for $x\in V$ the coordinates of $Ax$ with respect to said basis, that is $Ax = \sum_{i=1}^M a^{(x)}_i e_i$. Then the $\mu^N$-variance norm is given by
        \begin{equation*}
            \|x\|_{\mu^N\text{-cov}} 
            =\begin{cases}
                (a^{(x)})^T \Sigma^{-1} a^{(x)} & \text{if } a^{(x)} \in \operatorname{Im}(\Sigma), \\
                +\infty & \text{otherwise,}
            \end{cases}
        \end{equation*}
        where $\Sigma \in \bbR^{M\times M}$ is the empirical covariance matrix of the coordinates $a^{(y_1)}, \cdots, a^{(y_N)}$.
    \end{enumerate}
\end{proposition}
\begin{proof}
    We begin by proving $(ii)$, from which $(i)$ will follow. Since the covariance operator $\calK : V^* \to V$ is a finite rank operator, we have that $\calC(\calR_{\mu^N}) = \calK(V^*)$, where $\calC : \calR_{\mu^N} \to V$ is the extended covariance operator. Hence it follows from the expression
    \begin{equation*}
        \calK f = \frac{1}{N} \sum_{i=1}^N y_i f(y_i),
    \end{equation*}
    that $H_{\mu^N} = \calK(V^*) \subset \operatorname{span}\{y_1, \cdots, y_N\}$, which by \cref{TheoremCameronMartinNormDefinitionsCoincide} implies that $\|x\|_{\mu^N\text{-cov}}$ is infinite for all $x \notin \operatorname{span}\{y_1, \cdots, y_N\}$. Consequently we will only need to consider this finite span in the subsequent analysis.

    Next, observe that by writing $y_i = \sum_{m=1}^M a^{(y_i)}_m e_m$, we obtain the following expression for the functional covariance of $\mu^N$ for all $f\in V^*$
    \begin{align*}
        q(f,f) 
        &= \frac{1}{N}\sum_{i=1}^N f(y_i)^2 
        = \frac{1}{N}\sum_{i=1}^N \bigg(  \sum_{m=1}^M a^{(y_i)}_m f(e_m) \bigg)^2 
        = \frac{1}{N}\sum_{i=1}^N \langle a^{(y_i)}, b\rangle^2_{\bbR^M} 
        = \langle b, \Sigma b\rangle_{\bbR^M},
    \end{align*}
     with $b \in \bbR^M$ given by $b_m = f(e_m)$, and where $\Sigma \in \bbR^{M\times M}$ is the empirical covariance matrix of the coordinates $a^{(y_1)}, \cdots, a^{(y_N)}$. Conversely, if $b \in \bbR^M$ is fixed, then $f_b(x) := \langle a^{(x)}, b\rangle_{\bbR^M}$ defines a continuous linear functional on $\operatorname{span}\{y_1, \cdots, y_N\}$, which extends continuously to $V$ via Hahn-Banach. Consequently we obtain that
    \begin{align}\label{eqBanachEmpiricalVarianceNormEquation}
        \|x\|_{\mu^N\text{-cov}}^2
        &= \sup_{f \in V^* } \frac{f(x)^2}{q(f,f)} 
        = \sup_{b \in \bbR^M } \frac{\langle a^{(x)}, b\rangle^2_{\bbR^M}}{ \langle b, \Sigma b\rangle_{\bbR^M} }  
        =  \begin{cases}
            (a^{(x)})^T \Sigma^{-1} a^{(x)} & \text{if } a^{(x)} \in \operatorname{Im}(\Sigma), \\
            +\infty & \text{otherwise,}
            \end{cases}
    \end{align}
    where the last equality follows from \cref{exampleVariancenormRd}. The equality $H_\mu = \operatorname{span}\{y_1, \cdots, y_N\}$ then follows by \cref{TheoremCameronMartinNormDefinitionsCoincide} since $a^{(x)} \in \operatorname{Im}(\Sigma)$ if and only if $x \in \operatorname{span}\{y_1, \cdots, y_N\}$, which proves $(i)$.
\end{proof}

We want to stress that the choice of basis and projection map is purely for computational convenience, and will lead to the same result since the definition of the variance norm is basis-independent. The following corollary is a direct consequence of \cref{eqBanachEmpiricalVarianceNormEquation}, and relates the $\mu$-variance norm with respect to empirical measures on Banach spaces to the classical Mahalanobis distance in $\bbR^M$.

\begin{corollary}\label{corrVarNormBanachEmpirical}
    Under the assumptions of \cref{theoremEmpiricalBanachVarianceNorm}, we have that
    \begin{equation*}
        \|x\|_{\mu^N\text{-cov}} = \|a^{(x)}\|_{\nu-\text{cov}},
    \end{equation*}
    where $\nu = \frac{1}{N}\sum_{i=1}^N\delta_{a^{(y_i)}}$ is an empirical measure on $\bbR^M$.
\end{corollary}

While \cref{corrVarNormBanachEmpirical} provides a way to compute variance norms using coordinates relative to a basis of the Cameron-Martin space $H_{\mu^N}$, the construction of such a basis depends heavily on the structure of the underlying space $V$. This is evident, for example, in the functional $L^2[0,1]$ Mahalanobis literature, where the Mahalanobis distance for a $d$-dimensional time series of length $T$ is computed by flattening the data and applying the standard Mahalanobis distance in $\bbR^{Td}$ \citep{2015FunctionalMahalanobisTruncated, 2020FunctionalMahalanobis}. From the perspective of \cref{theoremEmpiricalBanachVarianceNorm}, this corresponds to constructing a basis for the discretized paths. This approach works well due to the specific structure of the $L^2[0,1]$ inner product. However, it does not generalize to settings where the geometry of $V$ is fundamentally different. In such cases, it is unclear how to obtain a tractable algorithm without explicitly relying on a Hilbert space structure to facilitate computations.

\section{Specialization to Hilbert Spaces}\label{sectionHilbert}

In this section we specialize to the case where $V$ is a Hilbert space with inner product $\langle \cdot, \cdot \rangle$. In doing so, we are able to diagonalize the covariance operator $\calK$ of $\mu \in \calM_V$, to express the variance norm and Cameron-Martin space in terms of the eigenvalues of $\calK$. This generalizes the results of \cite{2020FunctionalMahalanobis} from the setting $L^2[0,1]$ of functional data analysis to any separable Hilbert space, without any assumptions of continuity of a stochastic process or injectivity of covariance operator. This generalization is necessary to obtain a theory consistent for empirical measures, which by definition gives rise to non-injective covariance operators, and to obtain a general algorithm for computations which takes into account the infinite-dimensional properties of the chosen space $V$.

\subsection{Hilbert Space Characterization}\label{subsectionTheoryHilbert}

Recall from \cref{lemma_qfg_equals_fCg} that
\(
    \langle g, \calK f \rangle = q(g,f) = q(f,g) = \langle f, \calK g \rangle,
\)
for all $f,g\in V$, and that $\calK$ is compact. Consequently, this implies that $\calK$ is a symmetric, positive, compact operator, hence by the spectral theorem \cite[see e.g.][Theorem 28.3]{FunctionalAnalysisPeterDLax} there exists an orthonormal sequence of eigenvectors $(e_n)_{n=1}^\infty$ and non-negative eigenvalues $(\lambda_n)_{n=1}^\infty$ such that
\begin{equation}\label{eqDiagonalizeCovarianceOperator}
\calK f = \sum_{n=1}^\infty \lambda_n \langle f, e_n \rangle e_n, \quad \forall f \in V.
\end{equation}
If $V$ is finite-dimensional, we instead replace $(e_n)_{n=1}^\infty$ by a finite collection. Our ultimate goal is to derive an explicit computational formula for the variance norm $\|\cdot\|_{H_\mu}$. To achieve this, we first need to characterize the Cameron-Martin space $H_\mu$, as \cref{TheoremCameronMartinNormDefinitionsCoincide} states that this is the subspace of $V$ where the variance norm is finite. Understanding $H_\mu$ will also be crucial when we later consider empirical measures constructed from observed data in \cref{sectionComputingVarNorms}.

\begin{theorem}\label{theoremHilbertCameronMartin}
    Let $\mu \in \calM_V$ for a Hilbert space $V$, and let $(e_n)_{n=1}^\infty$ and $(\lambda_n)_{n=1}^\infty$ be orthonormal eigenvectors and eigenvalues that diagonalize the covariance operator $\calK : V \to V$. Then the Cameron-Martin space $H_\mu$ is given by
    \begin{align*}
         H_\mu 
         &= \bigg\{ h\in V : \sum_{n=1,\, \lambda_n\neq0}^\infty \frac{\langle h, e_n \rangle^2}{\lambda_n} < \infty  \text{ and } \bigg( \forall n\geq 1,\,  \lambda_n=0 \implies \langle h, e_n\rangle = 0\bigg) \bigg\} 
         = \calK^{\frac{1}{2}}(V), \nonumber
    \end{align*}
    where $\calK^{\frac{1}{2}}$ is the square root operator of $\calK$. The variance norm is given by
    \begin{equation}\label{eqHilbertSqrtKNorm}
        \|h\|_{\mu\text{-cov}}^2 =
        \begin{cases}
            \|\calK^{-\frac{1}{2}} h \|^2 = 
            \sum_{n=1,\, \lambda_n\neq0}^\infty \frac{\langle h, e_n \rangle^2}{\lambda_n} & \text{if } h \in \calK^{\frac{1}{2}}(V), \\ 
            \infty & \text{otherwise.}
        \end{cases}
    \end{equation}
\end{theorem}
\begin{proof}
    Let $h\in H_\mu$, that is $h = \calC k = \lim_n \calK f^{(n)}$ for some $k\in \calR_\mu$ and a sequence $f^{(n)} \in V^*$ such that $\|k-f^{(n)}\|_{L^2(\mu_\frakm)} \to 0$. Using the symmetry of $\calK$ we obtain that
    \begin{align*}
        \langle h, e_j\rangle 
        = \lim_n \langle \calK f^{(n)}, e_j\rangle 
        = \lim_n \langle f^{(n)}, \lambda_j e_j\rangle,
    \end{align*}
    for all $j\geq1$, and hence $\langle h, e_j\rangle= 0$ whenever $\lambda_j = 0$. Moreover, \cref{lemma_qfg_equals_fCg} implies that $\|f^{(n)}\|_{L^2(\mu_\frakm)} = \| \calK^{\frac{1}{2}}f^{(n)}\|$ for all $n\geq 1$, and consequently we find that
    \begin{align*}
        \|h\|_{H_\mu}^2
        &=  \| k\|^2_{L^2(\mu_\frakm)}
        = \lim_n \| f^{(n)}\|_{L^2(\mu_\frakm)}^2
        = \lim_n \| \calK^{\frac{1}{2}} f^{(n)}\|^2 \\
        &= \lim_n \sum_{j=1}^\infty \lambda_j \langle f^{(n)}, e_j\rangle^2 
        \geq  \sum_{j=1}^\infty \lim_n \lambda_j \langle f^{(n)}, e_j\rangle^2 \\
        &= \sum_{j=1,\, \lambda_j\neq0}^\infty \lim_n \frac{\langle \calK f^{(n)}, e_j \rangle^2}{\lambda_j}
        = \sum_{j=1,\, \lambda_j\neq0}^\infty \frac{\langle h, e_j \rangle^2}{\lambda_j},
    \end{align*}
    where the inequality follows by Fatou's lemma. This shows that $h \in \calK^\frac{1}{2}(V)$ since $\|h\|_{H_\mu}^2 < \infty$.

    Conversely, let $h \in \calK^\frac{1}{2}(V)$. We want to show that $h = \calC k$ for some $k\in \calR_\mu$. We do this by defining the sequence
    \(
        f^{(n)} := \sum_{j=1, \lambda_j\neq 0}^n \frac{\langle h, e_j \rangle}{\lambda_j} e_j,
    \)
    and noting that $\calK f^{(n)} = \sum_{j=1}^n \langle h,e_j\rangle e_j \to h$ as $n\to\infty$. Furthermore, $f^{(n)}$ converges in $L^2(V, \mu_\frakm)$ to some element $k\in R_\mu$ since we have the bound
    \begin{align*}
        \| f^{(n)} - f^{(m)} \|_{L^2(\mu_\frakm)}^2
        &= \sum_{j=n,\, \lambda_j\neq0}^m \frac{\langle h, e_j\rangle^2}{\lambda_j} \leq \sum_{j=n,\, \lambda_j\neq0}^\infty \frac{\langle h, e_j\rangle^2}{\lambda_j},
    \end{align*}
    for all $n<m$. Consequently we obtain that $h = \calC k = \lim_n \calK(f^{(n)})$, from which \eqref{eqHilbertSqrtKNorm} follows by definition.
\end{proof}

\begin{remark}
    Note that the expression for the variance norm is $\|h\|_{\mu\text{-cov}} = \sqrt{\langle h, \calK^{-1}h\rangle }= \| \calK^{-\frac{1}{2}}h \|$, analogous to the finite-dimensional $\bbR^d$ case.
\end{remark}

In the functional data analysis literature \citet{2020FunctionalMahalanobis} argued that the naive functional Mahalanobis distance $\| \calK^{-\frac{1}{2}}h \|$ fails to be defined due to the non-invertibility of the square root operator $\calK^\frac{1}{2}$ in the $L^2[0,1]$ setting. However, when viewed through the lens of variance norms and Cameron-Martin spaces as per \cref{theoremHilbertCameronMartin}, we can see how such a notion can still be made precise despite the difficulties present in the infinite-dimensional and singular settings by allowing the anomaly distance to be infinite if the covariance structure of the underlying distribution does not match the new samples. This point of view was for instance taken by \citet{2023LyonsShao} for their conformance score anomaly distance in the finite-dimensional setting.

\subsection{Regularized Variance Norms}\label{subsecRegularizedVarianceNorms}

A classical result from Gaussian probability theory states that $\mu(H_\mu) = 0$ whenever $\mu$ is a Gaussian measure and $\dim(H_\mu) = \infty$ \cite[see e.g.][Theorem 2.4.7]{bogachevGaussianMeasures}. This means that the sample outcomes of a $V$-valued random variable will almost surely not lie in the Cameron-Martin space $H_\mu$, making the variance norm infinite with probability one. This issue was addressed in the functional data analysis literature in the special case $V = L^2[0,1]$ by regularizing the functional Mahalanobis distance, under the assumptions of a continuous covariance function and an injective covariance operator \citep{2020FunctionalMahalanobis}. Using our framework we are able to extend these results to the general Hilbert space setting without these restrictive assumptions.

There are two equivalent viewpoints for how to obtain said regularization. First, recall by \cref{theoremHilbertCameronMartin} that the $\mu$-variance norm of $x\in V$ is given by $\|x\|_{\mu\text{-cov}} = \|\calK^{-\frac{1}{2}}x\|$ if $x\in \operatorname{Im}(\calK^{\frac{1}{2}})$, and infinity otherwise. The first definition of a regularized norm is obtained by replacing the inverse $\calK^{-\frac{1}{2}}$ with the \textit{Tikhonov regularized operator} $R_\alpha = (\calK + \alpha I)^{-1}\calK^{\frac{1}{2}}$ with smoothing parameter $\alpha > 0$. Tikhonov regularization is a classical tool used in statistics (e.g.\ ridge regression) and functional analysis to deal with ill-posed equations \citep[see e.g.][]{kress2013LinearIntegralEquations}. In contrast to the inverse $\calK^{-\frac{1}{2}}$, the Tikhonov operator $R_\alpha$ is well-defined on all of $V$ and is an approximation of the pseudo-inverse of $\calK^\frac{1}{2}$.

\begin{definition}\label{defRegularizedVarianceNorm}
    Let $V$ be a Hilbert space, and let $\mu \in \calM_V$. We define the $\alpha$-regularized $\mu$-variance norm with smoothing parameter $\alpha>0$ as
    \begin{equation}\label{eqRegularizedVarianceNorm}
        \|x\|_{\mu, \alpha} := \big\| (\calK + \alpha I)^{-1} \calK^{\frac{1}{2}}x \big\|,
    \end{equation}
    for $x\in V$, where $\calK$ is the covariance operator of $\mu$. 
\end{definition}

The alterative definition, following \cite{2020FunctionalMahalanobis}, is based on the idea to approximate each $x\in V$ by an element $x_\alpha \in H_\mu$, $\alpha>0$, and then take the $\mu$-variance norm of $x_\alpha$, which is finite by construction. Since no closest element in $H_\mu$ to $x$ exists in the infinite-dimensional case (since $H_\mu$ might not closed in $V$), $x_\alpha$ is chosen by minimizing
    \begin{align}\label{eqArgmin1}
        x_\alpha 
        &:= \argmin_{h\in H_\mu} \|x-h\|^2 + \alpha \|h\|^2_{H_\mu}.
    \end{align}
Because $H_\mu$ is a reproducing kernel Hilbert space as discussed in the remark proceeding \cref{TheoremCameronMartinNormDefinitionsCoincide}, it follows from \citet[Theorem 8.4]{LearningTheoryCuckerZhou} that the unique solution $x_\alpha$ to \eqref{eqArgmin1} is given by
\begin{align*}
    x_\alpha 
    = (\calK + \alpha I)^{-1}\calK x 
    = \sum_{n=1}^\infty \frac{\lambda_n}{\lambda_n + \alpha} \langle e_n, x \rangle e_n,
\end{align*}
where $(e_n)_{n=1}^\infty$ and $(\lambda_n)_{n=1}^\infty$ are the eigenvectors and eigenvalues of $\calK$. Moreover, \cref{theoremHilbertCameronMartin} implies that the squared $\mu$-variance norm of $x_\alpha$ is
\begin{equation}\label{varianceNorm}
    \|x_\alpha\|_{\mu\text{-cov}}^2 = \|\calK^{-\frac{1}{2}}x_\alpha\|^2 = \| (\calK + \alpha I)^{-1} \calK^{\frac{1}{2}}x \|^2 = \sum_{n=1}^\infty \frac{\lambda_n}{(\lambda_n + \alpha)^2}\langle e_n, x \rangle^2,
\end{equation}
coinciding with the Tikhonov perspective of \cref{defRegularizedVarianceNorm}. Note that the element $x_\alpha$ depends not only on $\alpha$ but also on $\calK$, which depends on $\mu$. We will use the notation $\|x\|_{\mu, \alpha}$, rather than $\|x_\alpha\|_\mu$, which better highlights this dependence. This is important when working with empirical variance norms $\|x\|_{\mu^N, \alpha}$ based on a finite sample of data drawn from $\mu$.

\subsection{Computing Variance Norms and Kernelization}\label{sectionComputingVarNorms}

For most machine learning applications, the underlying probability measure $\mu$ is not explicitly known, and we must base our models on finite samples assumed to be drawn from $\mu$. A natural estimator of the underlying distribution is the empirical measure $\mu^N = \frac{1}{N} \sum_{i=1}^N \delta_{x_i}$. In this subsection, we derive computational formulas for the variance norm with respect to $\mu^N$, and show how it is directly related to kernelization via Reproducing Kernel Hilbert Spaces (RKHS).

In this subsection, we denote by $\mu^N$ the empirical measure of a sample, and $\calK_N$ the covariance operator of $\mu^N$, which we call the empirical covariance operator. The following result follows directly from \cref{theoremEmpiricalBanachVarianceNorm,theoremHilbertCameronMartin} given the fact that $\calK_N = \frac{1}{N}\sum_{i=1}^N (x_i-\widehat{\frakm}) \big\langle \cdot, x_i-\widehat{\frakm}\big\rangle$ is a finite rank operator.

\begin{proposition}\label{theoremCameronMartinEmpiricalMeasure}
    The Cameron-Martin space of $\mu^N$ is given by
    \begin{align*}
        H_{\mu^N} &= \operatorname{span}\{x_1-\widehat{\frakm}, \cdots, x_N-\widehat{\frakm}\} 
              = \operatorname{span}\{e_1, \cdots, e_M\},
    \end{align*}
    where $\widehat{\frakm} = \frac{1}{N}\sum_{i=1}^N x_i$ is the empirical mean of the data, and $e_1, \cdots, e_M$ are the eigenvectors of $\calK_N$ with positive eigenvalues. The empirical $\alpha$-regularized variance norm is given by
    \begin{align*}
        \|z\|_{\mu^N, \alpha}^2 = \sum_{m=1}^M \frac{\lambda_m}{(\lambda_m + \alpha)^2} \langle z, e_m\rangle^2,
    \end{align*}
    for $z\in V$.
\end{proposition}

The following theorem presents an algorithm for computing the eigenvalues and eigenvectors of the empirical covariance operator, based on an SVD decomposition of the inner product Gram matrix. This concept is closely related to the techniques used for kernel PCA \citep{1998KernelPCA}. We however give our own original proof, and relate the results back to the Cameron-Martin space of the empirical measure.

\begin{theorem}\label{theoremEigenvectorsCovOperator}
    Let $A \in \bbR^{N\times N}$ be defined by $A_{i,j} = \langle f_i, f_j\rangle$, where $f_i = \frac{x_i-\widehat{\frakm}}{\sqrt{N}}$, with SVD decomposition $A = U \Sigma U^T$. Let $v^{(n)}$ be the $n$-th column of $U$, and $\lambda_n = \Sigma_{n,n}$. Define $M = \max\{m\leq N : \lambda_m > 0\}$. Then the elements defined by
    \begin{align}\label{eqSpectralDecompEigenvalueDef}
        e_n = \sum_{i=1}^N v^{(n)}_i f_i,
    \end{align}
    are orthogonal eigenvectors of $\calK_N$ with eigenvalues $\lambda_{n}$ and norms $\|e_n\| = \sqrt{\lambda_n}$ for $n\leq M$, and $e_n = 0$  for $M < n \leq N$. Moreover, we have that $\operatorname{span}\{e_1, \cdots, e_M\} = H_{\mu^N}$.
\end{theorem}
\begin{proof}
    To prove that $e_m$ is an eigenvector of $\calK_N$ with eigenvalue $\lambda_m$ for $m\in \{1, \cdots, M\}$, observe that
    \begin{align*}
        \calK_N \bigg( \sum_{i=1}^N v_i^{(m)} f_i \bigg) 
        = \sum_{j=1}^N f_j \bigg\langle \sum_{i=1}^N v_i^{(m)} f_i, f_j \bigg\rangle 
        = \sum_{j=1}^N f_j  \sum_{i=1}^N v_i^{(m)} \langle f_i, f_j \rangle 
        = \sum_{j=1}^N f_j \lambda_m v_j^{(m)},
    \end{align*}
    where we used that $\sum_{i=1}^N v_i^{(m)} \langle f_i, f_j \rangle = (Av^{(m)})_j = \lambda_m v^{(m)}_j$.

    Next, we verify that the vectors $\{ e_1, ..., e_N \}$ are linearly independent. We see that
    \begin{align*}
        \langle e_n, e_m \rangle
        &= \bigg\langle  \sum_{i=1}^N v_i^{(m)} f_i, \sum_{j=1}^N v_j^{(n)} f_j \bigg\rangle 
        = \sum_{i=1}^N \sum_{j=1}^N v_i^{(m)} v_j^{(n)} \langle f_i, f_j\rangle\\
        &= \langle v^{(n)}, A v^{(m)} \rangle_{\bbR^N} 
        = \lambda_m \langle v^{(n)}, v^{(m)} \rangle_{\bbR^N},
    \end{align*}
    for $n,m \in \{ 1, \cdots, N\}$. When $n\neq m$ we have that $\langle v^{(n)}, v^{(m)}\rangle_{\bbR^N} = 0$, hence $ \langle e_n, e_m \rangle = 0$. For $1 \leq m \leq M$ the element $e_m$ is non-zero since $\langle e_m, e_m \rangle = \lambda_m \langle v^{(m)}, v^{(m)}\rangle_{\bbR^N} > 0$. On the other hand, if $M < n \leq N$ then $\lambda_n = 0$, hence $\langle e_n, e_n \rangle = 0$ and consequently $e_n = 0$. 
    
    Finally, we want to use \cref{theoremCameronMartinEmpiricalMeasure} to conclude that $H_{\mu^N} = \operatorname{span}\{e_1, \cdots, e_M\}$. To this end, define $f = (f_1, \cdots, f_N)$ and $e = (e_1, \cdots, e_N)$ as column vectors. Using this notation, \eqref{eqSpectralDecompEigenvalueDef} can be written as $f = U e \iff e = U^T f$, from which it follows that $\operatorname{span}\{f_1, \cdots, f_N\} = \operatorname{span}\{e_1, \cdots, e_N\} = \operatorname{span}\{e_1, \cdots, e_M\}$. This concludes the proof.
\end{proof}

\cref{theoremHilbertCameronMartin} implies that the variance norm depends only on the choice of inner product on $V$, as well as the eigenvectors of the covariance operator. For empirical measures, \cref{theoremEigenvectorsCovOperator} provides an algorithm for computing these based on a finite sample of data, given an inner product. For applications, the special case where $V$ is an RKHS  \citep{2001KernelLearningCuckerSmale, kernelLearningScholkopfSmola} is of great interest, which we briefly introduce below before we present our final algorithm for computing variance norms.


\begin{definition}
    Let $\calX$ be a set, and let $\calH$ be a Hilbert space of functions of $\calX$. A reproducing kernel is defined as a function $k : \calX \times \calX \to \bbR$ satisfying
    \begin{enumerate}[label=(\roman*)]
        \item $\forall x\in\calX,\,\, k(\cdot, x) \in \calH$,
        \item $\forall x\in\calX,\,\, \forall f\in \calH,\,\, f(x) = \big\langle f, k(\cdot, x)\big\rangle_\calH$.
    \end{enumerate}
    Furthermore, $\calH$ is said to be a \textit{reproducing kernel Hilbert space} (RKHS) if there exists a reproducing kernel for $\calH$. 
\end{definition}

From a machine learning perspective, it is often helpful to view RKHSs through the lens of \textit{feature maps}. A feature map is defined as a function $\phi : \calX \to \calF$, where $\calF$ is a Hilbert space. Every such feature map induces a positive definite kernel via $k(x,y) := \langle \phi(x), \phi(y)\rangle_\calF$. Conversely, given a RKHS with reproducing kernel $k$, the canonical feature map $\phi(x):= k(\cdot, x)$ reproduces $k$ in the sense that $\langle \phi(x), \phi(y) \rangle_\calH = \langle k(\cdot, x), k(\cdot, y) \rangle_\calH = k(x,y)$. A kernelized variance norm is obtained by lifting an initial measure $\mu$ using a feature map $\phi$ via $\nu := \mu \circ \phi^{-1}$. The $\nu$-variance norm is then computed in the RKHS associated with $\phi$. This formulation is advantageous for applications because it allows for the use of \textit{kernel tricks}: inner products in $\calF$ can be computed via kernel evaluations $k(x, y)$, even when the explicit form of $\phi(x)$ is unavailable or infinite-dimensional.

The generality of \cref{theoremEigenvectorsCovOperator}, which only assumes that $V$ is a Hilbert space, allows us to apply it directly in the kernelized setting. Let ${x_1, \dots, x_N} \subset \calX$ be a dataset and $\phi: \calX \to \calF$ a feature map into a Hilbert space $\calF$. We define the empirical measure $\mu^N := \frac{1}{N} \sum_{i=1}^N \delta_{\phi(x_i)}$ and let $V$ be the RKHS induced by $\phi$. In this setting, the Gram matrix of inner products becomes the kernel Gram matrix. This gives a solid theoretical foundation for the kernelized Mahalanobis distance \citep{2001KernelizedMahalanobisIEEE} within our unified framework. Note that the non-kernelized (linear) setting is recovered by taking $\phi = I$, the identity map, and we refer to this case as the \textit{linear kernel}.

Because the feature map $\phi$ may be non-linear and possibly infinite-dimensional, direct computation of inner products between normalized elements (e.g., $\langle f_n, f_m \rangle$) may not be feasible. To address this, we express these inner products as linear combinations of kernel evaluations $\langle x_i, x_j \rangle = k(x_i, x_j)$. Specifically, let $N$, $f_i$, $M$, $v_i^{(m)}$, $e_m$, and $\lambda_m$ be as defined in \cref{theoremEigenvectorsCovOperator}. Then by bilinearity of the inner product:
\begin{align}\label{eqAlgo1}
\langle f_m, f_n\rangle
= \frac{1}{N} \bigg( \langle x_m, x_n\rangle
- \frac{1}{N}\sum_{j=1}^N \langle x_m, x_j\rangle
- \frac{1}{N}\sum_{j=1}^N \langle x_n, x_j\rangle
+ \frac{1}{N^2} \sum_{j=1}^N \sum_{i=1}^N \langle x_i, x_j\rangle \bigg),
\end{align}
and for any test point $h$:
\begin{align}\label{eqAlgo2}
\left\langle \frac{e_m}{\sqrt{\lambda_m}}, h \right\rangle
= \sum_{i=1}^N \frac{v^{(m)}_i}{\sqrt{\lambda_m N}} \left( \langle x_i, h \rangle
- \frac{1}{N} \sum_{j=1}^N \langle x_j, h \rangle \right).
\end{align}
Furthermore, we use a simple dynamic programming procedure to avoid a naive $\mathcal{O}(N^4)$ and $\mathcal{O}(N^2)$ time complexity when computing \eqref{eqAlgo1} and \eqref{eqAlgo2}, respectively. The full procedure for computing the kernelized Mahalanobis distance and its nearest-neighbour variant (conformance score) is detailed in \crefrange{algoGramMatrix}{algoConformance}. The fitting procedure described in \cref{algoGramMatrix} has time complexity $\mathcal{O}(N^2(K + N))$, where $K$ is the time complexity of a single kernel evaluation. For inference, detailed in \cref{algoMahalanobis,algoConformance}, both the kernelized Mahalanobis distance and conformance score can be computed in $\mathcal{O}(N(K + M))$ time, where $M \leq N$ is the number of eigenvalues.

\begin{algorithm}[h!]
    \DontPrintSemicolon
    \caption{Kernelized Gram matrix w.r.t. $\mu^N = \frac{1}{N}\sum_{i=1}^N\delta_{\phi(x_i)}$.}\label{algoGramMatrix}
    
    \KwInput{Data $\{x_1, \cdots, x_N \}$. }

    \tcp{Compute normalized Gram matrix via kernel trick}

    $B_{i,j} \leftarrow \big\langle \phi(x_i), \phi(x_j) \big\rangle = k(x_i, x_j)$ for $i,j \in \{1, \cdots, N\}$

    Column mean $a_i \leftarrow \frac{1}{N}\sum_{j=1}^N B_{i,j}$ for $i \in \{1, \cdots, N\}$
    
    Matrix mean $b \leftarrow \frac{1}{N}\sum_{i=1}^N a_i$

    $A_{i,j} \leftarrow \frac{1}{N}( B_{i,j} - a_i - a_j + b)$ for $i,j \in \{1, \cdots, N\}$
    \tcp*{$A_{i,j} = \langle f_i, f_j \rangle$}
    
    \tcp{Compute inner products of eigenvectors and data}
    Compute SVD decomposition $U \Sigma U^t = A$

    Set $M \leftarrow \max\{ m\leq N : \Sigma_{m, m} > \lambda \}$

    \For{$n=1$ \textup{\textbf{to}} $N$}{
        \For{$m=1$ \textup{\textbf{to}} $M$}{
            $E_{n,m} \leftarrow \sum_{i=1}^N \frac{U_{i,m}}{\sqrt{N \Sigma_{m,m}}}( B_{i,n} - a_n)$
            \tcp*{$E_{n,m} = \langle \frac{e_m}{\sqrt{\lambda_m}}, \phi(x_n) \rangle$}
        }
    }
    \KwOutput{Matrix $E$, and SVD decomposition $A = U \Sigma U^t$.}
\end{algorithm}

\begin{algorithm}[h!]
    \DontPrintSemicolon
    \caption{Kernelized Mahalanobis distance w.r.t. $\mu = \frac{1}{N}\sum_{i=1}^N\delta_{\phi(x_i)}$.}\label{algoMahalanobis}
    \KwInput{Data $\{x_1, \cdots, x_N \}$. SVD decomposition matrices $U, \Sigma \in \bbR^{N\times N}$ and $E \in \bbR^{N\times M}$ as per Algorithm \ref{algoGramMatrix}. Regularization $\alpha>0$. A new sample $y$. }

    \tcp{Compute inner product of eigenvectors and sample}
    Use kernel trick $s_i \leftarrow \langle \phi(y), \phi(x_i) \rangle = k(y, x_i)$ for $i \in \{1, \cdots, N\}$
    
    Average $ r \leftarrow \frac{1}{N}\sum_{i=1}^N s_i$
    
    $p_m \leftarrow \frac{1}{\sqrt{N\Sigma_{m,m}}} \sum_{i=1}^N U_{i, m}(s_i-r)$ for $m \in \{1, \cdots, M\}$
    \tcp*{$p_m = \big\langle \frac{e_m}{\sqrt{\lambda_m}}, \phi(y) \big\rangle$}

    \tcp{Calculate Mahalanobis distance}
    Average $c_m \leftarrow \frac{1}{N} \sum_{i=1}^N E_{i, m}$ for $m \in \{1, \cdots, M\}$
    \tcp*{$c_m = \big\langle \frac{e_m}{\sqrt{\lambda_m}},  \frac{1}{N}\sum_{n=1}^N \phi(x_n)\big\rangle$}

    $d \leftarrow \sqrt{\sum_{m=1}^M \frac{\Sigma_{m,m}}{(\Sigma_{m,m}+\alpha)^2}(p_m-c_m)^2}$
    \tcp*{$d = \|\phi(y)-\frac{1}{N}\sum_{n=1}^N \phi(x_n)\|^2_{\mu^N, \alpha}$}

    \KwOutput{Kernelized Mahalanobis distance $d$ with $\alpha$-regularization.}
\end{algorithm}

\begin{algorithm}[h!]
    \DontPrintSemicolon
    \caption{Kernelized conformance score w.r.t. $\mu = \frac{1}{N}\sum_{i=1}^N\delta_{\phi(x_i)}$.}\label{algoConformance}
    \KwInput{Data $\{x_1, \cdots, x_N \}$. SVD decomposition matrices $U, \Sigma \in \bbR^{N\times N}$ and $E \in \bbR^{N\times M}$ as per Algorithm \ref{algoGramMatrix}. Regularization $\alpha>0$. A new sample $y$. }

    \tcp{Compute inner product of eigenvectors and sample}
    Use kernel trick $s_i \leftarrow \langle \phi(y), \phi(x_i) \rangle = k(y, x_i)$ for $i \in \{1, \cdots, N\}$
    
    Average $ r \leftarrow \frac{1}{N}\sum_{i=1}^N s_i$
    
    $p_m \leftarrow \frac{1}{\sqrt{N\Sigma_{m,m}}} \sum_{i=1}^N U_{i, m}(s_i-r)$ for $m \in \{1, \cdots, M\}$
    \tcp*{$p_m = \big\langle \frac{e_m}{\sqrt{\lambda_m}}, \phi(y) \big\rangle$}

    \tcp{Calculate nearest-neighbour Mahalanobis distance}
    $d_n \leftarrow \sum_{m=1}^M \frac{\Sigma_{m,m}}{(\Sigma_{m,m}+\alpha)^2}(p_m - E_{n,m})^2$ for $n \in \{1, \cdots, N\}$
    \tcp*{$d_{n} = \|\phi(y)-\phi(x_n)\|^2_{\mu^N, \alpha}$}
    
    $c \leftarrow \sqrt{\min_n d_n}$

    \KwOutput{Kernelized conformance score $c$ with $\alpha$-regularization.}
\end{algorithm}

\clearpage

\subsection{Consistency, Speed of Convergence, and Distributional Results}

We now turn to the statistical properties of the sample estimator $\|x\|_{\mu^N, \alpha}$. Similar results were obtained by \citet{2020FunctionalMahalanobis} in the special case $V = L^2[0,1]$ under the assumptions that $\calK$ is injective. Most of their proofs easily generalize to our setting which we instead base on variance norms of measures $\mu\in\calM_V$ for general separable Hilbert spaces $V$. The consistency analysis and speed of convergence relies on the following lemma, which follows from standard properties of covariance operators on Hilbert spaces \citep[see e.g.][Theorems 8.1.1 and 8.1.2]{2015TheoreticalFoundationsOfFunctionalDataAnalysisWithAnIntroductionToLinearOperators}.

\begin{lemma}\label{lemmaCovarianceOperatorConvergence}
    Let $\mu \in \calM_V$ and $\mu^N$ be an empirical measure of $\mu$, with covariance operators $\calK$ and $\calK_N$, respectively. Then $\|\calK_N - \calK\|_{op} \to 0$ as $N\to\infty$ almost surely. Furthermore, if $\mu$ has finite fourth moment, then $\|\calK_N - \calK \|_{op}= O_P(N^{-\frac{1}{2}})$.
\end{lemma}

The following theorem proves that the empirical regularized variance norm converges to the actual regularized variance norm almost surely. Furthermore, if $\mu$ has finite fourth moment, we obtain a speed of convergence of $N^{-\frac{1}{4}}$.

\begin{theorem}\label{theoremConsistentEstimator}
Let $x\in V$ and $\mu \in \calM_V$. Then the empirical regularized $\mu$-variance norm is consistent almost surely, that is,
\begin{equation*}
    \|x\|_{\mu^N, \alpha} \to \|x\|_{\mu, \alpha},
\end{equation*}
as $N\to\infty$. Additionally, if $\mu$ has finite fourth moment, then the speed of convergence in probability is
\begin{equation*}
    \|x\|_{\mu^N, \alpha} - \|x\|_{\mu, \alpha} = O_P(N^{-\frac{1}{4}}).
\end{equation*}
\end{theorem}
\begin{proof}
    Fix $x\in V$ and $\alpha>0$. For brevity we write $T^\alpha_N = (\calK_N + \alpha I)^{-1}$ and $T^\alpha = (\calK + \alpha I)^{-1}$. By the reverse triangle inequality we obtain that
    \begin{align}\label{eqConsistency1}
        \bigg|\|x\|_{\mu^N, \alpha} - \|x\|_{\mu, \alpha} \bigg|
        &= \bigg|\| T^\alpha_N \calK_N^{\frac{1}{2}}x\| - \| T^\alpha \calK^{\frac{1}{2}}x\| \bigg|
        \leq \bigg\|   T_N^\alpha \calK_N^{\frac{1}{2}}x - T^\alpha \calK^{\frac{1}{2}}x    \bigg\| \nonumber\\
        &\leq \| T_N^\alpha \|_{op} \|\calK_N^{\frac{1}{2}}x - \calK^{\frac{1}{2}}x\| + \|T_N^\alpha - T^\alpha\|_{op} \|\calK^{\frac{1}{2}}x\|.
    \end{align}
     First note that $T^\alpha_N$ and $T^\alpha$ are bounded by $\| T^\alpha_N\|_{op} \leq \frac{1}{\alpha}$. \cref{lemmaCovarianceOperatorConvergence} implies that $\|\calK_N - \calK\|_{op} \to 0$ almost surely, from which it follows that the first term of \eqref{eqConsistency1} goes to $0$ as $N\to\infty$. The second term also goes to $0$, by \citet[Corollary 8.3]{BasicClassesOfLinearOperators}, since $\calK - \calK_N = (\calK + \alpha I) - (\calK_N + \alpha I)$. This proves consistency.

     As for the speed of convergence, \citet[Corollary 8.2]{BasicClassesOfLinearOperators} implies that 
     \begin{align*}
         \| T_N^\alpha - T \|_{op} \leq \frac{\|T^\alpha\|_{op}^2 \|\calK_N - \calK \|_{op} }{1 - \|T^\alpha\|_{op} \|\calK_N - \calK \|_{op}}
     \end{align*}
     which is of order $O(\|\calK_N - \calK \|_{op})$ as $N\to\infty$. Furthermore, since $\calK_N$ and $\calK$ are positive operators, we have that $ \|\calK_N^{\frac{1}{2}} - \calK^{\frac{1}{2}}\|_{op} \leq \|\calK_N - \calK \|_{op}^\frac{1}{2}$ \cite[see e.g.][Theorem X.1.1]{1997MatrixAnalysisBhatia}. Combining this with \cref{lemmaCovarianceOperatorConvergence} we obtain a speed of convergence in probability of $O_P(N^{-\frac{1}{4}})$.
\end{proof}

When $\mu$ is a Gaussian measure, we are able to obtain an explicit distribution of the Mahalanobis distance as an infinite sum of independent chi-squared random variables. This generalizes the classical Hotelling's T-statistic and the Gaussian functional Mahalanobis distance case \citep{2020FunctionalMahalanobis} to the general Hilbert space setting.

\begin{theorem}\label{theoremGaussianDistribution} 
    Let $\mu$ be a Gaussian measure on a Hilbert space $V$, and let $\alpha>0$. If $X\sim \mu$ is drawn from the measure $\mu$, then the squared $\alpha$-regularized Mahalanobis distance has distribution
    \begin{equation*}
        \|X-\frakm\|^2_{\mu, \alpha} \stackrel{d}{=} \sum_{n=1}^\infty \bigg(\frac{\lambda_n}{\lambda_n + \alpha}\bigg)^2 Y_n,
    \end{equation*}
    where $\frakm$ is the mean of $\mu$, $\lambda_n$ are the eigenvalues of the covariance operator $\calK$ of $\mu$, and $Y_1, Y_2, \cdots$ is a sequence of i.i.d. standard $\chi^2_1$ random variables.
\end{theorem}
\begin{proof}
    Let $(e_n)_{n=1}^\infty$ be an orthonormal sequence of eigenvectors of $\calK$. It follows by \eqref{varianceNorm} that
    \begin{equation*}
        \|X-\frakm\|^2_{\mu, \alpha} = \sum_{n=1}^\infty \frac{\lambda_n}{(\lambda_n + \alpha)^2} \langle e_n, X-\frakm \rangle^2.
    \end{equation*}
    Since $\mu$ is a Gaussian measure, the vectors $e_n$ are by definition Gaussian distributed when acting as continuous linear functionals on $V$. Moreover, we have that
    \begin{align*}
        \E\big[\langle e_n, X-\frakm \rangle\langle e_m, X-\frakm \rangle\big] = \langle e_n, \calK e_m\rangle = \begin{cases}
            \lambda_n & \text{if } n=m,\\
            0 & \text{otherwise},
        \end{cases}
    \end{align*}
    and
    $
        \E\big[\langle e_n, X-\frakm \rangle\big] = \langle e_n, \E[X]\rangle - \langle e_n, \frakm \rangle= 0,
    $
    from which the result follows.
\end{proof}

\section{Nearest Neighbour Properties}\label{sectionNN}

As discussed in \cref{subsecMahalConf}, for some applications it may be advantageous to measure outlier distances via nearest-neighbours rather than the Mahalanobis distance to the mean. In this section, we study some useful properties of the infinite-dimensional nearest- and furthest-neighbour $\mu$-variance norm in the Hilbert space setting. Let $X_0, ..., X_N \sim \mu$ be i.i.d. samples. We are interested in studying the random empirical variance norm nearest-neighbour distance, defined as
\begin{equation}\label{eqRandomNNDistance}
    \min\limits_{1 \leq i \leq N} \|X_0 - X_i\|_{\mu^N,\alpha}.
\end{equation}
We adopt a stepwise approach by analyzing three progressively more complex cases:
\begin{enumerate}
    \item The reference point $X_0$ belongs to the corpus defining the empirical measure $\mu^N$.
    \item $X_0$ is out of corpus and the norm measured w.r.t.\ $\mu$; this requires $\alpha$-regularization.
    \item $X_0$ is out of corpus and the norm measured w.r.t.\ the empirical measure $\mu^N$. 
\end{enumerate}

\subsection{Case 1: Reference point in corpus}
We begin by considering deterministic sample points $\{x_1, ..., x_N \}$ defining an empirical measure $\mu^N$, and simplify \eqref{eqRandomNNDistance} by setting $x_0 = x_1$ and leaving out $x_1$ from the calculation of the minimum. In this case, one can work with the non-regularized variance norm, since $x_0$ will lie in the empirical Cameron-Martin space. In the high-dimensional setting, an interesting property emerges: if all $x_1, ..., x_N$ are linearly independent, then without regularization every pair of distinct points in the corpus is equidistant under the empirical variance norm, at distance exactly $\sqrt{2N}$. The use of the nearest-neighbour Mahalanobis distance is uninformative in this case. We formalize this in the following:

\begin{proposition}\label{propEmpiricalVarianceNormSqrt2NEquidistant}
    Let $\{ x_1, ..., x_N\} \subset V$ be linearly independent, and let $\mu^N = \frac{1}{N}\sum_{i=1}^N \delta_{x_i}$ be the empirical measure. Then for any $i \neq j$ we have
    \begin{align*}
        \sqrt{2N}\left(\frac{\lambda_{N-1}}{\lambda_{N-1}+\alpha}\right)\leq \|x_j - x_i\|_{\mu^N, \alpha} \leq \sqrt{2N}\left(\frac{\lambda_{1}}{\lambda_{1}+\alpha}\right),
    \end{align*}
where $\lambda_{m}$, $m\geq 1$, are the eigenvalues of $\calK_N$ in decreasing order. In particular, when $\alpha=0$, we obtain that $\|x_j - x_i\|_{\mu^N} = \sqrt{2N}$.
\end{proposition}
\begin{proof}
Recall from \cref{theoremCameronMartinEmpiricalMeasure} that
$
    \|\cdot\|^2_{\mu^N, \alpha} =  \sum_{m=1}^M \frac{\lambda_m}{(\lambda_m+\alpha)^2} \langle \cdot, z_m\rangle^2,
$
where $\lambda_m$ and $z_m$ are the eigenvalues and (normalized) eigenvectors of the covariance operator $\calK_N$ of $\mu^N$. \cref{theoremEigenvectorsCovOperator} says these are given by an SVD decomposition: Let $A \in \bbR^{N\times N}$ be defined by $A_{i,j} = \langle f_i, f_j\rangle$, where $f_i = \frac{x_i-\widehat{\frakm}}{\sqrt{N}}$, with SVD decomposition $A = U \Sigma U^T$. Let $v^{(n)}$ be the $n$-th column of $U$, and $\lambda_n = \Sigma_{n,n}$. Define $M = \max\{m\leq N : \lambda_m > 0\}$, and $e_n = \sum_{i=1}^N v^{(n)}_i f_i$. The vectors $z_m = e_m / \sqrt{\lambda_m}$ are orthonormal eigenvectors of $\calK_N$.
First note for all $j, m,$ that
    \begin{align*}
    \langle f_j,e_m\rangle &= 
    \left\langle f_j, \sum_{k=1}^N v^{(m)}_k f_k\right\rangle 
    = \sum_{k=1}^N v^{(m)}_k \langle f_j,  f_k\rangle 
    = (Av^{(m)})_j = \lambda_m v_j^{(m)},
\end{align*}
hence
\begin{align*}
    \|x_j - x_i \|_{\mu^N, \alpha}^2 
    &= N\|f_j - f_i \|_{\mu^N, \alpha}^2 
    = N \sum_{m=1}^M \frac{\lambda_m}{(\lambda_m+\alpha)^2} \langle f_j - f_i, \frac{e_m}{\sqrt{\lambda_m}}\rangle^2 \\
    &= N \sum_{m=1}^M \frac{\lambda_m^2}{(\lambda_m+\alpha)^2}\left( v_j^{(m)} -v_i^{(m)} \right)^2. 
\end{align*}
Next, we have that $M=N-1$ due to linear independence and mean-centering. Consequently, we have that $v^{(N)} = (\frac{1}{\sqrt{N}}, ..., \frac{1}{\sqrt{N}})$ which follows from the unit vector belonging to the null space of $A$. Thus
\begin{align*}
    N \sum_{m=1}^M \left( v_j^{(m)} -v_i^{(m)} \right)^2 
    = N \left( \| v_j - v_i\|^2 - (v_j^{(N)}-v_i^{(N)})^2\right) 
    = N(2-0) = 2N,
\end{align*}
where $v_j$ denotes the $j$-th row of $U$, and the conclusion follows from the monotonicity of $\lambda \mapsto \frac{\lambda}{\lambda + \alpha}$.
\end{proof}

In the finite-dimensional $V = \bbR^d$ case, \cref{propEmpiricalVarianceNormSqrt2NEquidistant} is only applicable when the sample size $N$ satisfies $N\leq d$. However, when $V$ is infinite-dimensional, for many relevant distributions, sample points will be linearly independent with probability one, and all corpus points will be equidistant without regularization. This suggests that the in-corpus case is ill-suited for studying the infinite-dimensional setting, and provides additional justification for using $\alpha$-regularization. As such, we consider an alternative approach in the sequel.

\subsection{Case 2: Reference point out of corpus, sub-Gaussian measure}

We now consider the \textit{out-of-corpus} case, where the random reference point $X_0$ and corpus points $X_1, ..., X_N$ are drawn identically distributed from a sub-Gaussian measure $\mu$. Our goal is to analyse the concentration properties of the difference between the nearest- and furthest-neighbor distances in the infinite-dimensional setting. For ease of notation, we use the symbol $\lesssim$ to denote inequality up to a universal constant. We use the following infinite-dimensional definition of sub-Gaussianity, which is a special case of an \textit{$R$-sub-Gaussian} random variable with respect to covariance operators \citep{1997RSubGaussianity}.

\begin{definition}\label{defSubGaussianHilbert}
    A random variable $X \sim \mu \in \calM_V$ with covariance operator $\calK$ is said to be sub-Gaussian with respect to $\calK$ if there exists a $\beta \geq 0$ such that for all $z\in V$
    \begin{equation}\label{eqSubGaussianHilbert}
        \E \left[  e^{\langle z, X-\E X\rangle}\right] \leq e^{\beta^2\langle \calK z, z\rangle}.
    \end{equation}
    Moreover, the sub-Gaussian norm of $X$ with respect to $\calK$ is defined as the smallest constant $\beta \geq 0$ such that \eqref{eqSubGaussianHilbert} holds, denoted $\|X\|_{\psi_2,\calK}$. $X$ is said to be $\calK$-Gaussian if \eqref{eqSubGaussianHilbert} is an equality with $\beta=1$.
\end{definition}

Our analysis relies on a Hilbert-space version of the classical Hanson-Wright inequality \citep{2021HansonWrightHilbertSpaces}. To apply this result, we need to furthermore impose a mild Bernstein-like tail condition on the squared norm of our random variables, as follows:

\begin{definition}\label{defBernsteinCond}
    A sub-Gaussian random variable $X \sim \mu \in \calM_V$ with covariance operator $\calK$ and sub-Gaussian norm $\beta = \|\mu\|_{\psi_2, \calK}$ is said to satisfy a Bernstein condition on the squared norm with respect to $\calK$ if
    \begin{equation}\label{eqBernsteinCond}
        \E \left| \|X\|^2-\E\|X\|^2 \right|^k \lesssim k! \beta^{k-2}\|\calK\|_{op}^{k-2} \|\calK\|^2_{HS}
    \end{equation}
    for all $k \geq 3$.
\end{definition}

The following definitions of effective rank $\textbf{r}(\calK)$ and dimension $\textbf{d}(\calK)$ for covariance operators $\calK$, defined below, will play an important role in our analysis. The effective rank $\textbf{r}(\calK)$ was previously used by \citet{2017ConcentrationInequalitiesAndMomentBoundsForSampleCovarianceOperators} in the infinite-dimensional setting, see also \citet{2012VershyninIntroductionToTheNonAsymptoticAnalysisOfRandomMatrices}. Another well-studied measure of rank in matrix theory is the so-called \textit{stable rank} of a matrix (see \citet{2025StableRankAndIntrinsicDimensionOfRealAndComplexMatrices} and references therein). The effective dimension $\textbf{d}(\calK)$ below can be obtained as the squared quotient of the stable rank and effective dimension, and has previously been used in the context of particle systems in physics under the name of \textit{participation ratio} \citep{2022AScaleDependentMeasureÓfSystemDimensionalityPHYSICS, 1993LocalizationTheoryAndExperimentPHYSICS}. These quantities naturally appear in our analysis as a byproduct of the Hanson-Wright inequality \citet{2021HansonWrightHilbertSpaces}. Note that if $X$ is a $d$-dimensional isotropic Gaussian with covariance matrix $\sigma^2 I_d$, then the effective dimension and rank is exactly $d$.

\begin{definition}\label{defEffectiveDimensionAndRank}
    Let $\mu \in \calM_V$ with covariance operator $\calK$. We define the effective dimension $\textbf{d}(\calK)$ and effective rank $\textbf{r}(\calK)$ by
    \begin{equation*}
        \textbf{d}(\calK) = \frac{\operatorname{Tr}(\calK)^2}{\|\calK\|_{HS}^2} = \frac{(\sum_{i=1}^\infty \lambda_i)^2}{\sum_{i=1}^\infty \lambda_i^2}, \qquad \qquad 
        \textbf{r}(\calK) = \frac{\operatorname{Tr}(\calK)}{\|\calK\|_{op}} = \frac{\sum_{i=1}^\infty \lambda_i}{\lambda_1}
    \end{equation*}
    where $\lambda_m$ are the eigenvalues of $\calK$ in decreasing order.
\end{definition}

We first consider the general Hilbert space case with norm $\|\cdot\|$, and then specialize to the $\mu$-variance norm via a transformation. The following result shows that with probability $1-\delta$, the relative difference between the furthest- and nearest-neighbour distance is bounded above by a term proportional to $\sqrt{\log(N/\delta)}$, divided by the effective rank or dimension of the underlying data. Therefore, if $\log(N/\delta) \lesssim \textbf{r}(\calK)$ and $\log(N/\delta) \lesssim \textbf{d}(\calK) $, then the difference between the furthest- and nearest-neighbour distance is small. This result provides a theoretical justification for why nearest-neighbour methods can remain effective in infinite-dimensional settings where one might expect random points to become nearly equidistant: The concentration phenomenon is not governed by the ambient infinite dimension of $V$, but rather by the effective dimensionality of the covariance operator $\calK$.

\begin{proposition}\label{propMinMaxDifference}
    Let $\delta \in (0,1)$. If $X_0, X_1, ..., X_N$ are drawn i.i.d. from a centered sub-Gaussian measure $\mu$ satisfying the Bernstein condition \eqref{eqBernsteinCond}, then with probability $1-\delta$
    \begin{align*}
    \frac{\max\limits_{1 \leq i \leq N}\|X_0-X_i\|^2 - \min\limits_{1 \leq i \leq N}\|X_0-X_i\|^2}{\E \|X_0 - X_1\|^2}
    & \lesssim \beta^2 \epsilon(N, \delta, \calK)
    \end{align*}
    where
    \begin{align*}
        \epsilon(N, \delta, \calK) =  \max\left\{ \sqrt{\frac{ \log(2N/\delta)}{\textbf{d}(\calK)}},\; \frac{\log(2N/\delta)}{\textbf{r}(\calK)} \right\}.
    \end{align*}
\end{proposition}
\begin{proof}
Let $D_i^2 = \| X_0 - X_i\|^2$. First note that
$
    \E [D_i^2] = \E \left( \|X_0\|^2 + \|X_i\|^2 -  2\langle X_0, X_i\rangle \right)  = 2 \operatorname{Tr}(\calK).
$
By the Hanson-Wright inequality in Hilbert spaces \cite[Theorem 2.8]{2021HansonWrightHilbertSpaces}, there exists a universal constant $C>0$ such that for any $t>0$
\begin{align*}
    P\left( \left| D_i^2 - \E[D_i^2] \right| \ge t \right) \le 2 \exp\left( -C \min\left\{ \frac{t^2}{\beta^4\|\calK\|_{HS}^2}, \frac{t}{\beta^2\|\calK\|_{op}} \right\} \right).
\end{align*}
We want to use a union bound to obtain a concentration result. Let $\calA_i = \{ |D_i^2 - \E [D_i^2]| \geq t\}$. Consider the inequality
\begin{align*}
    P(\bigcup_{i=1}^N \calA_i) 
    \leq N P(\calA_1)
    \leq  N 2 \exp\left( -C \min\left\{ \frac{t^2}{\beta^4\|\calK\|_{HS}^2}, \frac{t}{\beta^2\|\calK\|_{op}} \right\}\right) 
    \leq \delta.
\end{align*}
We want this probability to be at most $\delta$, so by rearranging the terms we obtain this happens when
\begin{align*}
    t\gtrsim \beta^2\|\calK\|_{HS} \sqrt{\log\frac{2N}{\delta}}, \qquad \textnormal{and} \qquad  t\gtrsim \beta^2 \|\calK\|_{op}\log\frac{2N}{\delta}
\end{align*}
This gives, with probability $1-\delta$ and for all $1\leq i \leq N$, that
\begin{align*}
    |D_i^2 - \E [D_i^2] | \lesssim \beta^2 \max\left\{\|\calK\|_{HS} \sqrt{\log\frac{2N}{\delta}},\; \|\calK\|_{op}\log\frac{2N}{\delta} \right\}.
\end{align*}
The result then immediately follows by considering the furthest-neighbour distance minus the nearest-neighbour distance, divided by the relative expected size $E [D_i^2] =2\operatorname{Tr}(\calK)$.
\end{proof}

For the Tikhonov-regularized variance norm, we have the identity $\|z\|_{\mu, \alpha} = \|\calS_\alpha z\|$ for all $z\in V$, where $\calS_\alpha = (\calK+\alpha I)^{-1}\sqrt{\calK}$ (see \cref{subsecRegularizedVarianceNorms}). This allows us to reframe the problem: instead of considering a random variable $X \sim \mu$ with a regularized norm, we can equivalently study the transformed variable $Y = \calS_\alpha X$ with the standard Hilbert space norm. The distribution of $Y$ is given by the pushforward measure $\mu_\alpha = \mu \circ \calS_\alpha^{-1}$, and $Y$ is $\calK_\alpha$-sub-Gaussian with respect to its own covariance operator $\calK_\alpha$ if $X$ is $\calK$-sub-Gaussian (simply apply \eqref{eqSubGaussianHilbert} to the transformed measure). The covariance operator $\calK_\alpha$ is given by
\begin{equation*}
    \calK_{\alpha}
    = \calS_\alpha \calK \calS_\alpha^*
    = (\calK + \alpha I)^{-2}\calK^2 ,
\end{equation*}
which has the spectral decomposition $\calK_{\alpha} = \sum_{i=1}^\infty \left(\frac{\lambda_i}{\lambda_i + \alpha}\right)^2 \langle e_i, \cdot \rangle$. The effective dimension and rank of this operator are then
\begin{align*}
    \textbf{d}(\calK_{\alpha}) 
    = \frac{\operatorname{Tr}(\calK_{\alpha})^2}{\|\calK_{\alpha}\|_{HS}^2}
    = \frac{\left( \sum_{i=1}^\infty \left(\frac{\lambda_i}{\lambda_i+\alpha}\right)^2 \right)^2}{\sum_{i=1}^\infty \left(\frac{\lambda_i}{\lambda_i+\alpha}\right)^4}, \qquad 
    \textbf{r}(\calK_{\alpha}) 
    = \frac{\operatorname{Tr}(\calK_{\alpha})}{\|\calK_{\alpha}\|_{op}} 
    = \frac{\sum_{i=1}^\infty \left(\frac{\lambda_i}{\lambda_i + \alpha}\right)^2}{\left(\frac{\lambda_1}{\lambda_1 + \alpha}\right)^2}.
\end{align*}

\begin{corollary}\label{corrMinMaxDifferenceVarianceNorm}
Let $\alpha, \delta>0$. If $X_0, X_1, ..., X_N$ are drawn identically distributed from a centered sub-gaussian measure $\mu$, and if $\mu_\alpha$ satisfies the Bernstein condition \eqref{eqBernsteinCond} with sub-Gaussian constant $\beta_\alpha$, then with probability $1-\delta$
\begin{align*}
    \frac{\max\limits_{1 \leq i \leq N}\|X_0-X_i\|^2_{\mu, \alpha} - \min\limits_{1 \leq i \leq N}\|X_0-X_i\|^2_{\mu, \alpha}}{\E \|X_0 - X_1\|^2_{\mu, \alpha}}
    \lesssim \beta_\alpha^2 \epsilon(N, \delta, \calK_{\alpha}).
\end{align*}
\end{corollary}
\begin{remark}
    If $\mu$ is $\calK$-Gaussian, then $\mu_\alpha$ is $\calK_\alpha$-Gaussian, and hence both satisfy the Bernstein conditions, and $\beta = \beta_\alpha = 1$ \citep[see e.g.][]{2021HansonWrightHilbertSpaces}.
\end{remark}

\subsection{Case 3: Reference point out of corpus, empirical measure}

We now proceed to study the case where we take the variance norm with respect to the empirical measure $\mu^N$. In our analysis, we will use the following concentration result from \citet[Theorem 9]{2017ConcentrationInequalitiesAndMomentBoundsForSampleCovarianceOperators} applied to centered square integrable Hilbert space valued random variables.

\begin{lemma}\label{lemmaCovarianceConcentration}[\cite{2017ConcentrationInequalitiesAndMomentBoundsForSampleCovarianceOperators}]
Let $X, X_1, ..., X_N$ be i.i.d. square integrable centered random vectors with covariance operator $\calK$. If $X$ is $\calK$-sub-Gaussian, then for all $t\geq 1$, with probability at least $1- e^{-t}$,
\begin{equation*}
    \|\calK - \calK_N\|_{op} \lesssim \|\calK\|_{op} \max\left( \sqrt{\frac{\textbf{r}(\calK)}{N}}, \frac{\textbf{r}(\calK)}{N}, \sqrt{\frac{t}{N}}, \frac{t}{N} \right).
\end{equation*}
\end{lemma}

In the next proposition, we obtain a result analogous to \cref{propMinMaxDifference}, but with an additional error term due to the use of empirical measures. This error term is of order $\calO\left( \sqrt{\frac{\log(N)}{N}}\right)$ as $N \to \infty$, assuming $\delta$ and $\alpha$ constant.

\begin{proposition}
    Let $\alpha >0,$ $\delta\in(0,1)$, and let $X_0, X_1, ..., X_N$ be drawn i.i.d.\ from a measure $\mu$ satisfying the assumptions of \cref{propMinMaxDifference} and \cref{corrMinMaxDifferenceVarianceNorm}. Then with probability $1-\delta$ the empirical $\mu^N$-variance norm satisfies
    \begin{align*}
    \frac{\max\limits_{1 \leq i \leq N}\|X_0-X_i\|^2_{\mu^N, \alpha} - \min\limits_{1 \leq i \leq N}\|X_0-X_i\|^2_{\mu^N, \alpha}}{\E \|X_0 - X_1\|^2_{\mu, \alpha}}
    & \lesssim \beta_\alpha^2\epsilon(N, \frac{\delta}{3}, \calK_{\alpha}) + \frac{\Delta(N, \frac{\delta}{3}, \calK)\operatorname{Tr}(\calK)}{\alpha^2 \operatorname{Tr}(\calK_\alpha)} \left(1 + \beta^2\epsilon(N, \frac{\delta}{3}, \calK)\right),
    \end{align*}
    where $\Delta(N, \delta, \calK) = \|\calK\|_{op} \max\left( \sqrt{\frac{\textbf{r}(\calK)}{N}}, \frac{\textbf{r}(\calK)}{N}, \sqrt{\frac{\log(1/\delta)}{N}}, \frac{\log(1/\delta)}{N} \right)$.
\end{proposition}
\begin{proof}
Let $D_i^2 = \|X_0 - X_i\|_{\mu, \alpha}^2$ and $D_{i,N}^2 =  \|X_0 - X_i\|_{\mu^N, \alpha}^2$. Consider the inequality
\begin{equation}\label{eqEmpiricalNNConc1}
    \max_{1 \leq i \leq N}D_{i,N}^2 - \min_{1 \leq i \leq N}D_{i,N}^2
     \leq \left( \max_{1 \leq i \leq N}D_i^2 - \min_{1 \leq i \leq N}D_i^2 \right) + 2 \max_{1 \leq i \leq N} \left| D_{i,N}^2 - D_i^2 \right|.
\end{equation}
The first term of \eqref{eqEmpiricalNNConc1} can be bounded by \cref{corrMinMaxDifferenceVarianceNorm} with probability $1-\delta/3$:
\begin{equation}\label{eqEmpiricalNNConc2}
    \max_{1 \leq i \leq N}D_i^2 - \min_{1 \leq i \leq N}D_i^2 \lesssim \beta_\alpha^2 \epsilon(N, \delta/3, \calK_{\alpha}) \operatorname{Tr}(\calK_\alpha).
\end{equation}
The second term of \eqref{eqEmpiricalNNConc1} captures the error from using the empirical covariance operator $\calK_N$ instead of $\calK$. Let $d_i = X_0 - X_i$. The squared norms are quadratic forms:
\begin{equation*}
    D_i^2 = \langle d_i, \calK(\calK + \alpha I)^{-2} d_i \rangle \quad \text{and} \quad D_{i,N}^2 = \langle d_i, \calK_N(\calK_N + \alpha I)^{-2} d_i \rangle.
\end{equation*}
Their difference is bounded by
\begin{align*}
    |D_{i,N}^2 - D_i^2| 
    &= \left| \langle d_i, \left(\calK_N(\calK_N + \alpha I)^{-2} - \calK(\calK + \alpha I)^{-2}\right) d_i \rangle \right| \\
    &\leq \|d_i\|^2 \left\| \calK_N(\calK_N + \alpha I)^{-2} - \calK(\calK + \alpha I)^{-2} \right\| \\
    &\leq \frac{3}{\alpha^2} \|d_i\|^2 \|\calK - \calK_N\|_{op},
\end{align*}
where the last inequality is due to the following bound: We first use the resolvent identity for invertible squared operators, $A^{-2} - B^{-2} = A^{-2}(B-A)B^{-1} + A^{-1}(B-A)B^{-2}$, where $A = \calK_N + \alpha I$ and $B = \calK + \alpha I$, to write
\begin{align*}
    \calK_N(&\calK_N+\alpha I)^{-2}-\calK(\calK+\alpha I)^{-2} \\
    &= \calK_N\big((\calK_N+\alpha I)^{-2}-(\calK+\alpha I)^{-2}\big) + (\calK_N-\calK)(\calK+\alpha I)^{-2} \\
    &= \calK_N(\calK_N+\alpha I)^{-2}(\calK-\calK_N)(\calK+\alpha I)^{-1}
        +\calK_N(\calK_N+\alpha I)^{-1}(\calK-\calK_N)(\calK+\alpha I)^{-2} \\
    &\qquad\qquad\qquad\qquad\qquad\qquad\qquad\qquad\qquad +(\calK_N-\calK)(\calK+\alpha I)^{-2}.
\end{align*}
Taking operator norms and using $\|(\calK+\alpha I)^{-1}\|_{op}\leq \frac{1}{\alpha}$ together with $\|\calK(\calK+\alpha I)^{-1}\|_{op}\leq 1$, we obtain
\begin{align*}
    \big\|&\calK_N(\calK_N+\alpha I)^{-2}-\calK(\calK+\alpha I)^{-2}\big\|_{op} \\
    &\leq \| \calK_N(\calK_N+\alpha I)^{-2}\|_{op} \frac{1}{\alpha}\|\calK-\calK_N\|_{op} + \| \calK_N(\calK_N+\alpha I)^{-1}\|_{op}\frac{1}{\alpha^{2}}\|\calK-\calK_N\|_{op}
    + \frac{1}{\alpha^{2}}\|\calK-\calK_N\|_{op}\\
    &\leq \frac{1}{\alpha}\cdot\frac{1}{\alpha}\|\calK-\calK_N\|_{op} + 1\cdot\frac{1}{\alpha^{2}}\|\calK-\calK_N\|_{op} + \frac{1}{\alpha^{2}}\|\calK-\calK_N\|_{op}\\
    &= \frac{3}{\alpha^{2}}\|\calK-\calK_N\|_{op}.
\end{align*}
Thus, the second term of \eqref{eqEmpiricalNNConc1} is bounded by
\begin{equation*}
    \max_{1 \leq i \leq N} |D_{i,N}^2 - D_i^2| \leq \frac{3}{\alpha^2} \|\calK - \calK_N\|_{op} \max_{1 \leq i \leq N} \|X_0 - X_i\|^2.
\end{equation*}
Next, we apply concentration inequalities to bound $\|\calK - \calK_N\|_{op}$ and $\max_i \|X_0 - X_i\|^2$.
By \cref{lemmaCovarianceConcentration}, with probability at least $1-\delta/3$,
\begin{equation*}
    \|\calK - \calK_N\|_{op} \lesssim \Delta(N, \delta/3, \calK).
\end{equation*}
For the maximum norm term, first recall that $\E\|X_0 - X_i\|^2 = 2\operatorname{Tr}(\calK)$. By the same argument as in \cref{propMinMaxDifference} (Hanson-Wright inequality and a union bound over the $N$ vectors $d_i = X_0 - X_i$), we have with probability at least $1-\delta/3$ that
\begin{equation*}
    \max_{1 \leq i \leq N} \|X_0 - X_i\|^2 \lesssim 2\operatorname{Tr}(\calK) + \beta^2 \epsilon(N, \delta/3, \calK)\operatorname{Tr}(\calK).
\end{equation*}
Combining these bounds via a union bound (total probability $1-\frac{2}{3}\delta$), we get
\begin{equation}\label{eqEmpiricalNNConc3}
    \max_{1 \leq i \leq N} |D_{i,N}^2 - D_i^2| \lesssim \frac{\Delta(N, \frac{\delta}{3}, \calK)\operatorname{Tr}(\calK)}{\alpha^2}  \left(1 + \beta^2 \epsilon(N, \frac{\delta}{3}, \calK)\right).
\end{equation}
Combining \eqref{eqEmpiricalNNConc2} and \eqref{eqEmpiricalNNConc3} in \eqref{eqEmpiricalNNConc1} with yet another union bound (total probability $1-\delta$) and dividing by $\E \|X_0 - X_1\|^2_{\mu, \alpha}$ yields the final result.
\end{proof}

\section{Applications to Multivariate Time Series Novelty Detection}\label{sectionApplicationsToTimeSeries}

In this section, we apply the theory developed in \crefrange{sectionCamMartTheory}{sectionHilbert} to novelty detection of multivariate time series. Given a collection of non-anomalous multivariate time series, referred to as the normal corpus, we are presented with new time series samples that we want to classify as either belonging to the normal class or as outliers. We achieve this by defining an anomaly distance with respect to the corpus using either the Mahalanobis distance or the conformance score (see \cref{defMahalanobisDistance,defConformanceScore}), assuming that our data originates from a suitable Hilbert space. For time series, a natural choice of Hilbert space is the classical space $V = (L^2[0,1])^d$, as well as letting $V$ be the RKHS induced by our choice of kernel or feature map. The choice of $V$ directly affects how we measure similarity between elements in our normal corpus. For instance, for $V=L^2[0,1]$ two time series will be compared in a linear fashion by their $L^2[0,1]$ inner products. On the other hand, if $V$ is for example the RKHS of the signature kernel \citep{2021PdeSignatureKernel}, then the two time series are compared as \textit{rough paths} \citep{1998RoughPathTheory} in a non-linear fashion. Different choices of RKHS will lead to different ways to measure similarity between points, guided by the practitioner, by domain-specific knowledge or by empirical validation through techniques
like cross-validation.

For a given normal corpus $\{ x_1, \cdots, x_N\} \subset V$ we form either the standard empirical measure $\mu^N = \frac{1}{N}\sum_{i=1}^N \delta_{x_i}$, or the kernelized empirical measure $\mu^N = \frac{1}{N}\sum_{i=1}^N \delta_{\phi(x_i)}$ where $\phi$ is a feature map corresponding to a positive definite kernel whose kernel trick is known. The empirical measure $\mu^N$ can be interpreted as an estimator of the underlying distribution of the normal corpus. Given a new sample $y$, we proceed by calculating either the Mahalanobis distance 
$$ d_M(y; \mu^N) = \|y-\widehat{\frakm}\|_{\mu^N\text{-cov}},$$
where $\widehat{\frakm} = \frac{1}{N}\sum_{i=1}^N x_i$ is the mean of the normal corpus, or the conformance score 
$$ d_C(y; \mu^N) = \min_{1\leq n \leq N} \|y-x_n\|_{\mu^N\text{-cov}},$$
using \crefrange{algoGramMatrix}{algoConformance}, which defines an anomaly distance to the normal corpus. We then classify each new sample $y$ as an outlier or as belonging to the normal class based on a threshold $\gamma > 0$. 

The threshold $\gamma$ can be determined through various approaches. A data-driven method involves splitting the normal corpus into training and validation sets, then choosing $\gamma$ as an empirical quantile of the anomaly distances in the validation set. Alternatively, a theoretical approach uses the distribution of the Mahalanobis distance as outlined in \cref{theoremGaussianDistribution}, assuming the data follows a Gaussian distribution. If a labelled subset of outliers is available, a supervised learning approach can be employed, using $k$-fold cross-validation to determine an optimal threshold; however, this requires access to a supervised data set of outliers. In our experiments, we evaluate each anomaly distance using Precision-Recall (PR) AUC and ROC-AUC metrics, which consider sensitivity across all positive thresholds, thus eliminating the need to select a fixed threshold explicitly.

Although semi-supervised anomaly detection using the nearest-neighbour Mahalanobis distance, as opposed to the classical Mahalanobis distance, has been successfully employed in the finite-dimensional $\bbR^d$ setting \citep{2011MahalanobisDistanceKNNSemiconductor, 2020MahalanobisKNNHealthMonitoring, 2023LyonsShao, 2024PaolaSignatureAnomaly}, to our knowledge no comprehensive comparison of these two anomaly distances has been carried out in the literature. In this section we carry out an extensive comparison of our newly introduced kernelized conformance score (including the non-kernelized linear case) against the kernelized Mahalanobis distance for the task of semi-supervised time series novelty detection, using the infinite-dimensional framework developed in the previous sections.

\subsection{Time Series Kernels}\label{subsecTimeSeriesKernels}
We begin by giving a brief summary of the time series kernels considered in our experimentation. These consist of the linear kernel given by the $(L^2[0,1])^d$ inner product, a family of generalized integral-class kernels related to linear time warping \citep{2001DynamicTimeWarpingKernels}, and a collection of time-dynamic state-of-the-art time series kernels including the global alignment kernel \citep{2007GlobalAlignmentKernel, 2011FastGlobalAlignment}, the Volterra reservoir kernel \citep{2022GononVolterraReservoirKernel}, and signature kernels \citep{2019kernelsForSequentiallyOrderedData, 2021PdeSignatureKernel}.

\subsubsection{Static and Integral Class Kernels}\label{subsecIntegralClassKernels}
Consider first the natural Hilbert space $(L^2[0,1])^d$, where the inner product is given by
\begin{align}\label{eqIntegralClass1}
    \langle x, y\rangle = \int_{[0,1]} \langle x_t, y_t \rangle_{\bbR^d} dt,
\end{align}
for $x,y \in (L^2[0,1])^d$. If $x$ and $y$ are discretized on a regular time grid of size $T$, and consequently can be viewed as $d$-dimensional time series of length $T$, then the inner product \eqref{eqIntegralClass1} can simply be computed by flattening $x$ and $y$ into vectors in $\bbR^{Td}$, and then calculating their Euclidean dot product. Instead of using the Euclidean dot product, or in other words the linear kernel, one could replace this with any static kernel defined on $\bbR^{dT}$ as per \cref{defIntegralClassKernel} below. Similarly, we can also replace the linear static kernel in \eqref{eqIntegralClass1} to obtain a class of integral-type kernels with respect to a static kernel. In our experimentation we consider both flattened and integral-type kernels. We give the following definitions:

\begin{definition}\label{defIntegralClassKernel}
    We define a static kernel on $\bbR^d$ as a positive definite kernel $k: \bbR^d \times \bbR^d \to \bbR$. Given such $k$, we define the time series integral class kernel of $k$ to be the kernel
    \begin{align*}
        K_k(x,y) = \int_{[0,1]} k(x_t, y_t) dt,
    \end{align*}
    defined for $d$-dimensional time series $x$ and $y$.
\end{definition}

The positive definiteness of $K_k$ follows trivially from that of $k$. By replacing the Euclidean dot product with a possibly non-linear static kernel, an algorithm may be able to take certain non-linearities of the data into account to increase classification accuracies. The integral type kernels can in fact be seen as a variant of linear time warping kernels, which were first introduced by \citet{2001DynamicTimeWarpingKernels}. The static kernels we consider in this paper are:

\begin{enumerate}[label=(\roman*)]
    \item The linear kernel $k_{linear}(x,y) = \langle x, y\rangle$, which does not have any hyperparameters.
    \item The polynomial kernel $k_{poly}(x,y) = (c + \langle x ,y\rangle)^p$ with hyperparameters $c\in \bbR$ and $p \in \bbZ_+$.
    \item The RBF kernel $k_{RBF}(x,y) = e^{-\frac{|x-y|^2}{2\sigma^2}}$ with hyperparameter $\sigma > 0$.
\end{enumerate}

We thus have two distinct classes of time series kernels parametrized by static kernels: One is to consider flattened time series and static kernels in $\bbR^{Td}$, and the other is integral class kernels with respect to static kernels on $\bbR^d$. Note that the integral and static class kernels coincide when $k$ is the linear kernel, and are distinct otherwise. These time series kernels can be computed in $\calO(Td)$ time.

\subsubsection{Global Alignment Kernel}\label{subsecGlobalAlignment}

The global alignment kernel (GAK) \citep{2007GlobalAlignmentKernel, 2011FastGlobalAlignment} is a dynamic-time kernel which is able to take non-linear time lags into account when measuring the similarity of two time series via dynamic time warping. While the classical dynamic time warping fails to define positive definite kernels due to failing to satisfy the triangle inequality, the global alignment kernel is able to overcome this by summing over all possible global alignments of the time series.

\begin{definition}
    Let $x = (x_1, \cdots, x_T)$ and $y = (y_1, \cdots, x_L)$ be two time series of length $T$ and $L$ respectively. An alignment $\pi$, denoted $\pi \in \calA(T,L)$, is defined as a pair $\pi = (\pi_1, \pi_2)$ of vectors of length $p \leq T+L-1$ such that $1 = \pi_1(1) \leq \cdots \leq \pi_1(p) = T$ and $1 = \pi_2(1) \leq \cdots \leq \pi_2(p) = L$. Given a similarity measure $\varphi : \bbR^d\times\bbR^d \to [0,\infty)$, the cost $D_{x,y}(\pi)$ is defined as
    \begin{align*}
        D_{x,y}(\pi) := \sum_{i=1}^{|\pi|} \varphi(x_{\pi_1(i)}, y_{\pi_2(i)}),
    \end{align*}
    and the global alignment kernel is defined as
    \begin{align}
        K_{GA}(x,y) &= \sum_{\pi \in \calA(T,L)} e^{- D_{x,y}(\pi)}  
        = \sum_{\pi \in \calA(T,L)} \prod_{i=1}^{|\pi|} \kappa(x_{\pi_1(i)}, y_{\pi_2(i)}), \label{eqGAK}
    \end{align}
    where $\kappa = e^{-\varphi}$ is the local similarity.
\end{definition}

\citet{2007GlobalAlignmentKernel} proved that $K_{GA}$ defined via a local kernel $\kappa$ is positive definite if $\frac{\kappa}{1+\kappa}$ is positive definite. A sufficient condition for this is for $\kappa$ to be geometrically or infinitely divisible. In practice the local kernel $\kappa = \frac{k_{RBF}}{2 - k_{RBF}}$ is often used, and due to the exponential nature of \eqref{eqGAK} the GAK kernel is always made to be normalized in feature space via $\frac{K_{GA}(x,y)}{\sqrt{K_{GA}(x,x)K_{GA}(y,y)}}$.
The GAK kernel has a single hyperparameter $\sigma>0$ inherited from the static RBF kernel, and $K_{GA}(x, y)$ can be computed in $\calO(TLd)$ time using dynamic programming.

\subsubsection{Volterra Reservoir Kernel}\label{subsecVolterraReservoirKernel}

The Volterra Reservoir Kernel (VRK) \citep{2022GononVolterraReservoirKernel} is a universal dynamic kernel designed for sequences of arbitrary length. The kernel is built by constructing a state-space representation of the classical Volterra series expansions \citep{wiener:book, sandberg1983series, Boyd1985}, a series representation for analytic maps between sequences. As discussed in detail in \cite{2022GononVolterraReservoirKernel,RC13} this idea is closely related to the principle of reservoir computing \citep{maass1, Jaeger04} and associated kernels \citep{RC16}. The VRK kernel was recently shown to outperform the RBF, GAK, and signature kernels in a market forecasting task \citep{2022GononVolterraReservoirKernel}.

For sequences of length $T$ the VRK kernel
with hyperparameters $\tau \in \bbR$ and $\lambda \in (0, 1)$ is defined as
\begin{align*}
    K^{\textsc{Volt}}(x,y) 
    &= 1 + \sum_{k=1}^T \lambda^{2k} \prod_{t=0}^{k-1} \frac{1}{1-\tau^2 \langle x_{T-t}, y_{T-t}\rangle},
\end{align*}
for time series $x$ and $y$ of equal length  such that $\tau^2 \| x\|\|y\| < 1$, where $\langle \cdot, \cdot \rangle$ is the Euclidean inner product on $\bbR^d$. The solution can be computed in $\calO(Td)$ time using a recursive relation between the kernel at different time steps.

\subsubsection{Signature Kernels}\label{subsecSignatureKernel}

The signature kernel \citep{2019kernelsForSequentiallyOrderedData, 2021PdeSignatureKernel} is a positive definite kernel for sequential data based on tools originating from stochastic analysis and rough path theory \citep{1998RoughPathTheory}. It has many desirable theoretical properties such as invariance to time-reparametrization, universality, and characteristicness on compact sets. Algorithms using signature kernels have successfully been applied to a wide variate of fields since their inception, for instance in Bayesian forecasting \citep{2020SignatureBayesianGaussianProcesses}, hypothesis testing \citep{2021SignatureKernelMeanEmbeddings}, and for support vector machines \citep{2021PdeSignatureKernel, 2023RandomFourierSignatureFeatures} achieveing state-of-the-art accuracies.

Below we give a very brief construction of the signature kernel, which is defined as an inner product in the extended tensor algebra via the so-called \textit{signature transform}. For a detailed introduction to signature kernels we refer to the seminal papers by \citet{2019kernelsForSequentiallyOrderedData, 2021PdeSignatureKernel}, the review article by \citet{2023SignatureKernelReview}, and the book by \citet{2024CassSalviLectureNotes}.

\begin{definition}
    Let $\calH$ be a Hilbert space. The $m$-fold iterated integral of a bounded variation path $x \in BV([0,1], \calH)$ is recursively defined as
    \begin{align*}
        S_0(x) := 1, \qquad S_{m+1}(x) = \int_0^1 S_{m}(x) \otimes dx_t.
    \end{align*}
    We define the signature transform as the map
    \begin{align*}
        S : BV([0,1], \calH)& \to \prod_{m=0}^\infty \calH^{\otimes m} \\
        x& \mapsto \big( S_m(x)\big)_{m=0}^\infty,
    \end{align*}
    and similarly, we define the truncated signature as the map $S_{0:n}(x) := \big( S_m(x)\big)_{m=0}^n$. Here we use the convention that $\calH^{\otimes 0} = \bbR$.
\end{definition}

Given a path $x\in BV([0,1], \bbR^d)$ and a static kernel $k$ on $\bbR^d$, we may canonically lift $k$ to a path $k_x$ taking values in its RKHS $\calH$ via 
$
    t \mapsto k(x_t, \cdot) \in \calH
$
using the reproducing kernel property of $k$. If $k$ is the linear kernel $k(x,y) = \langle x, y\rangle_{\bbR^d}$, then $k_x$ is simply the original path $x$. However, if we choose $k$ to be a non-linear kernel such as the RBF kernel, then $k_x$ would genuinely be different to $x$, and in this particular case $k_x$ would take values in an infinite-dimensional Hilbert space where direct computations of truncated signature features are impossible. The main idea behind the signature kernel is to define the sequential kernel $k^{sig}$ w.r.t. a static kernel $k$ as the inner product of signature transforms $S(k_x)$ and $S(k_y)$ given two paths $x$ and $y$.

\begin{definition}
   Let $k$ be a static kernel on $\bbR^d$. We define the $k$-lifted signature kernel as the mapping $k^{sig} : BV([0,1], \bbR^d) \times BV([0,1], \bbR^d) \to \bbR$,
   \begin{align*}
       k^{sig}(x, y) = \sum_{m=0}^\infty \big\langle S_m(k_x), S_m(k_y) \big\rangle_{\calH^{\otimes m}},
   \end{align*}
   where $ \langle \cdot, \cdot \rangle_{\calH^{\otimes m}}$ is the Hilbert-Schmidt inner product defined as $\langle a, b \rangle_{\calH^{\otimes m}} = \sum_{i=1}^n \langle a_i, b_i \rangle_\calH$ for elements $a = a_1 \otimes \cdots \otimes a_m$ and $b = b_1 \otimes \cdots \otimes b_m$. The truncated singature kernel $k^{sig}_{0:n}$ is defined similarly using the truncated signature.
\end{definition}

There are currently three main algorithms for computing signature kernels, each of which come with their separate advantages and disadvantages. We list the methods below:
\begin{enumerate}
    \item If $\dim \calH = d < \infty$, then the truncated signature $S_{0:m}(x)$ can be computed exactly in $\calO(Td^m)$ time, when treating the time series $x$ as a piecewise linear path of length $T$. The truncated signature can then be computed by taking the inner product of the truncated signature. In practice this method is only applicable when $k$ is the trivial linear kernel, and when $d$ is very small (e.g. $d<5$) due to the exponential time complexity.

    \item The second method is due to \citet[Algorithms 3 and 6]{2019kernelsForSequentiallyOrderedData}, and takes advantage of the kernel trick to compute $k_{0:m}^{sig}(x,y)$ via a Horner-type scheme. This can be computed in $\calO(LTmd)$ time using a non-geometric approximation of $k^{sig}(x,y)$, or exactly in $\calO(LT(md + m^3))$ time when viewing $x$ and $y$ as piecewise linear paths of lengths $T$ and $L$ with state-space $\bbR^d$.

    \item The last method is due to \citet{2021PdeSignatureKernel}, who proved that the signature kernel $k^{sig}(x|_{[0,s]},y|_{[0,t]})$ solves the Goursat PDE
    \begin{align}\label{eqGoursatPDE}
        k^{sig}(x|_{[0,s]},y|_{[0,t]}) = 1 + \int_0^s \int_0^t  k^{sig}(x|_{[0,u]},y|_{[0,v]}) \langle dk_{x_u}, dk_{x_v} \rangle_{\bbR^d},
    \end{align}
    where $x|_{[0,s]}$ denotes the restriction of $x$ to the interval $[0, s]$. Equation \eqref{eqGoursatPDE} can be solved for piecewise linear paths using numerical PDE methods in $\calO(LTd)$ time \citep{2021PdeSignatureKernel}, but is often much slower than the truncated approaches.
\end{enumerate}

Generally the RBF-lifted signature kernel is preferred over the vanilla signature kernel. This is partly due to the latter having a tendency to blow up when the underlying time series are not properly normalized, something which is particularly pronounced for the PDE signature kernel which essentially acts as an inner product of tensor exponentials. In our experimentation we use the truncated signature kernel with the linear and RBF static kernels, as well as the RBF-lifted PDE signature kernel.

\vspace{10pt}

Another recently introduced variant of the signature and its signature kernel is the so-called \textit{randomized signature} \citep{2021ExpressivePowerOfRandomizedSignatures}, which we now define.

\begin{definition}\label{defRandomizedSignature}
    Let $M\geq 1$ be an integer. Fix an initial condition $z_0 \in \bbR^M$, random matrices $A_1, \cdots, A_d \in \bbR^{M\times M}$, random biases $b_1, \cdots, b_d \in \bbR^{M}$ and an activation function $\sigma$. The randomized signature $Z$ of $x \in BV([0,1], \bbR^d)$ is defined as the solution of the controlled differential equation (CDE)
    \begin{align}\label{eqRandomizedSignature}
        dZ_t = \sum_{i=1}^d \sigma(A_i Z_t + b_i)dx_t^{(i)}, \qquad Z_0 = z_0,
    \end{align}
    where $x^{(i)}$ denotes the $i$'th component of $x$. The randomized signature kernel is defined as the inner product of two randomized signatures.
\end{definition}

The randomized signature was first constructed by \citet{2021ExpressivePowerOfRandomizedSignatures} as a random projection of the signature, with an argument based on a non-trivial application of the Johnson-Lindenstrauss lemma. Randomized signatures have recently been successfully used for market anomaly detection \citep{2022CrpytoMarketAnomalyDetection}, graph conversion \citep{schäfl2023gsignatures}, optimal portfolio selection \citep{2023TeichmannOptimalPortfolioSelection, 2024cuchieroSignaturePortfolioTheory}, generative time series
modelling \citep{2024NiklasGononUniversalrandomisedsignaturesgenerative}, and for learning rough dynamical systems \citep{2023OnTheEffectivenessOfRandomizedSignaturesAsReservoirForLearningRoughDynamics}. The CDE \eqref{eqRandomizedSignature} has since been studied from the perspective of randomly initialized ResNets \citep{2023NeuralSignatureKernels, 2024NikolaTheoreticalDoundationsDeepSelectiveStateSpaceModels}, and path developments on compact Lie groups \citep{2023PCFGAN, 2024PathDevelopmentNetworkWithFiniteDimensionalLieGroup, 2024WillTurnerThomasCassFreeProbabilityPathDevelopments}. In our experiments, we use Gaussian random matrices and biases, with tanh activation function.

\subsection{Experiments}\label{subsecExperiments}
In this section we present an empirical study comparing the (potentially kernelized) Mahalanobis distance to the conformance score for semi-supervised multivariate time series novelty detection. Our primary objective is to validate \crefrange{algoGramMatrix}{algoConformance} presented in this paper for this infinite-dimensional setting. For comparisons of the finite-dimensional conformance score against other established methods like isolation forests, shapelets, and local outlier factors, we refer readers to \citet{2023LyonsShao}. Within the functional data analysis literature, the $(L^2[0,1])^d$ Mahalanobis distance has been evaluated against other common functional anomaly detection methodologies such as boxplots, outliergrams, and depth-based trimming \citep{2014Outliergram, 2020FunctionalMahalanobis}.

We will use UEA multivariate time series repository \citep{2018MultivateTimeSeriesRepository, 2021MultivariateBakeoff} in our experimentation, which in recent years has become a standard benchmark for multivariate time series classification. The repository contains 30 real world data sets consisting of multivariate time series, 26 of which are of equal lengths ranging from 8 to 2500 time steps, with state-space dimension ranging from 2 to 1345, see Table \ref{tab:data sets} for a summary. For the task of semi-supervised anomaly detection task we employ a one-versus-rest approach. In each experiment, we designate a single class label as the normal corpus, while considering all other classes as outliers. We evaluate the performance using both PR-AUC and ROC AUC. The results are then averaged across all class labels for a comprehensive assessment. Our experiment code is publically available at \url{https://github.com/nikitazozoulenko/kernel-timeseries-anomaly-detection}.

\begin{table}[]
\resizebox{0.823\textwidth}{!}{\begin{minipage}{\textwidth}
\begin{tabular}{l|l|l|l|l|l|l|l}
    \toprule
    Code & Name & Train size & Test size & Dims & Length & Classes & Avg. Corpus Size \\
    \midrule
    \midrule 
    AWR & ArticularyWordRecognition & 275 & 300 & 9 & 144 & 25 & 11 \\
    AF & AtrialFibrillation & 15 & 15 & 2 & 640 & 3 & 5 \\
    BM & BasicMotions & 40 & 40 & 6 & 100 & 4 & 10\\
    CR & Cricket & 108 & 72 & 6 & 1197 & 12 & 9\\
    DDG & DuckDuckGeese & 50 & 50 & 1345 & 270 & 5 & 10 \\
    EW & EigenWorms & 128 & 131 & 6 & 17,984 & 5 & 26\\
    EP & Epilepsy & 137 & 138 & 3 & 206 & 4 & 34\\
    EC & EthanolConcentration & 261 & 263 & 3 & 1751 & 4 & 65\\
    ER & ERing & 30 & 270 & 4 & 65 & 6 & 5\\
    FD & FaceDetection & 5890 & 3524 & 144 & 62 & 2 & 2945\\
    FM & FingerMovements & 316 & 100 & 28 & 50 & 2 & 158\\
    HMD & HandMovementDirection & 160 & 74 & 10 & 400 & 4 & 40\\
    HW & Handwriting & 150 & 850 & 3 & 152 & 26 & 6\\
    HB & Heartbeat & 204 & 205 & 61 & 405 & 2 & 102\\
    LIB & Libras & 180 & 180 & 2 & 45 & 15 & 12\\
    LSST & LSST & 2459 & 2466 & 6 & 36 & 14 & 176\\
    MI & MotorImagery & 278 & 100 & 64 & 3000 & 2 & 139\\
    NATO & NATOPS & 180 & 180 & 24 & 51 & 6 & 30\\
    PD & PenDigits & 7494 & 3498 & 2 & 8 & 10 & 749\\
    PEMS & PEMS-SF & 267 & 173 & 963 & 144 & 7 & 38 \\
    PS & PhonemeSpectra & 3315 & 3353 & 11 & 217 & 39 & 85\\
    RS & RacketSports & 151 & 152 & 6 & 30 & 4 & 38\\
    SRS1 & SelfRegulationSCP1 & 268 & 293 & 6 & 896 & 2 & 134\\
    SRS2 & SelfRegulationSCP2 & 200 & 180 & 7 & 1152 & 2 & 100\\
    SWJ & StandWalkJump & 12 & 15 & 4 & 2500 & 3 & 4\\
    UW & UWaveGestureLibrary & 120 & 320 & 3 & 315 & 8 & 15\\
    \bottomrule
\end{tabular}
\end{minipage}}
\caption{Summary of the 26 equal length UEA multivariate time series data sets. In our empirical study we consider all data sets of total size less than 8000, where the average corpus size is greater than 30.}\label{tab:data sets}
\end{table}

\subsubsection{Experimental Setup}
In total we will consider two different anomaly distances, namely the Mahalanobis distance and the conformance score, together with 11 different time series kernels. This includes the linear Euclidean kernel corresponding to the non-kernelized $(L^2[0,1])^d$ setting. In our open-source code we provide efficient PyTorch implementations of each kernel on both GPU and CPU, as well as an implementation of \crefrange{algoGramMatrix}{algoConformance} for computing the kernelized Mahalanobis distance and the kernelized conformance score. We consider the following time series kernels in our experimentation, all of which were defined in \cref{subsecTimeSeriesKernels}:

\begin{enumerate}
    \item The family of time series kernels obtained by flattening a given time series of length $T$ into a vector in $\bbR^{Td}$, and then applying a static kernel. We will use the RBF, polynomial, and linear kernels as our choices of static kernels, the latter of which corresponds to the $(L^2[0,1])^d$ inner product.
    
    \item The family of integral-class kernels (linear time warping), with RBF and polynomial static kernels.

    \item The global alignment kernel (GAK).

    \item The Volterra reservoir kernel (VRK).

    \item Four different variants of the signature kernel: The truncated signature kernel, the RBF-lifted truncated signature kernel, the RBF-lifted PDE signature kernel, and randomized signatures with tanh activation.
\end{enumerate}

Our work adds to the growing body of literature on anomaly detection using signature features, which was first studied in \citet{2023LyonsShao}. This was done by explicitly computing $m$-level truncated signature features, which has $\calO(Td^m)$ time complexity. Truncated signatures were later successfully used for market anomaly detection \citep{2022CrpytoMarketAnomalyDetection}, and radio astronomy \citep{2024PaolaSignatureAnomaly}. The use of signatures has however been limited to low-dimensional time series due to the exponential time complexity of explicitly computing truncated signatures. Our unified framework addresses this bottleneck, allowing for efficient computations of both signature conformance scores and signature Mahalanobis distances in $\calO(T^2 d)$ time. This significant improvement in $d$ opens up the use of these methods for high-dimensional time series data.

\subsubsection{Hyper-parameter Selection}
For each kernel, data set, and class label, we run an extensive grid search on the designated training set using repeated $k$-fold cross-validation with 4 folds and 10 repeats to find the optimal kernel hyper-parameters. Let $\bbR^d$ be the state-space, and let $T$ be the length of the time series for a given data set. For each method using the RBF static kernel we use the range $\sigma \in \frac{1}{\sqrt{d}}\{e^{-2}, e^{-1}, 1, e^{1}, e^{2}\}$, and similarly for the polynomial kernel we use $p \in \{2, 3, 4 \}$, and $c \in \{\frac{1}{4}, \frac{1}{2}, 1, 2, 4\}$. For the GAK kernel we use the previously specified $\sigma$ without the $\sqrt{d}$ term, multiplied by $\sqrt{T}\cdot\text{med}(\|x-y\|)$ as is recommended by \citet{2011FastGlobalAlignment}. For the VRK kernel we use $\tau \in \frac{1}{\sqrt{d}}\{\frac{1}{8}, \frac{1}{4}, \frac{1}{2}, 1 \}$, and we let $\lambda$ vary from $0.25$ to $0.999$ on an inverse logarithmic grid of size $10$. 

The signature kernels inherit hyper-parameters from their respective static kernels. We additionally scale the kernel-lifted paths by $s \in \frac{1}{\sqrt{d}}\{ \frac{1}{4}, \frac{1}{2}, 1, 2, 4\}$ for the truncated signature, and by $s \in \frac{1}{\sqrt{d}}\{ \frac{1}{8}, \frac{1}{4}, \frac{1}{2}, 1\}$ for the untruncated PDE kernel. We use lower values for the PDE signature kernel since untruncated signatures essentially can be viewed as tensor exponentials, whose inner products will blow up if the input values are too big. For the truncated signature kernel we let the truncation level be in $\{1, 2, 3, 4, 5, 6, 7\}$. For the randomized signature we use the tanh activation function with number of features in $\{ 10, 25, 50, 100, 200\}$, and random matrix variances taken from a logarithmic grid of size $8$ from $0.00001$ to $1$. Since the randomized signature is a randomized kernel, we perform the cross validation with $5$ different random seeds for the random matrix initializations, and take the best performing model (using the training set only), as is common practice for randomized methods.

Furthermore, for each method we also cross-validate over the Tikhonov regularization parameter $\alpha \in \{ 10^{-8}, 10^{-5}, 10^{-2}\}$, whether to concatenate time as an additional dimension to each time series, and the eigenvalue threshold $\lambda$ in \cref{algoGramMatrix}. We set an upper limit of 50 eigenvalues for the computation of the variance norm. For numerical stability, and to make the choices of $\alpha$ and $\lambda$ be comparable across all kernels and data sets, we normalize all time series kernels $K$ in feature space via $\frac{K(x,y)}{\sqrt{K(x,x)K(y,y)}}$.

\subsubsection{Pre-processing} 
For each data set and each class label, we normalize the data to have mean zero and standard deviation one, using the statistics of the normal corpus. Average-pooling is then performed to reduce the maximum length of all time series to 100 time steps. After this, we concatenate the zero vector to each time series to allow each dynamic kernel to be translation-sensitive, and we clip all values to be in $[-5, 5]$ for additional numerical stability. Furthermore, in our cross-validation we also include the choice of adding time as an additional dimension to all time series. For the VRK kernel specifically, we perform further clipping of the data based on the $\tau$ hyper-parameter, which is required to make the VRK kernel well-defined.

\subsubsection{Data and Results}

Due to the high computational cost of evaluating 11 time series kernels on all 26 UEA data sets, with up to 40 experiments per data set-kernel combination, each of which goes through an extensive repeated k-fold cross-validation, we focus our analysis on UEA data sets with a total size under 8,000 entries (see \cref{tab:data sets}). This excluded \textsc{PenDigits} and \textsc{FaceDetection}. Additionally, to ensure a sufficient statistical sample size, we only considered data sets where the average corpus size exceeded 30 entries, resulting in a final selection of 12 data sets.

The anomaly distances were computed as described by \crefrange{algoGramMatrix}{algoConformance}. The optimal kernel hyper-parameters were obtained separately for the Mahalanobis distance and the conformance score, via a 10 times repeated 4-fold cross-validation on the training data for each data set. The objective score used in the cross-validation was the sum of ROC-AUC and PR-AUC. When calculating the precision-recall metric, we let the non-outlier class be the positive class. The final model was then evaluated on the out-of-sample test set to obtain the final results, presented in  \cref{tab:scores_roc_auc} and \cref{tab:scores_pr_auc}.

\vspace{20pt}

\begin{table}[]
\resizebox{0.90\linewidth}{!}{\begin{minipage}{\linewidth}
    \begin{tabular}{lc||ccc|cc|c|c|cccc}
        \toprule
        \multirow{2}{*}{Data set}   &  \multicolumn{12}{c}{ROC AUC} \\
        \cline{3-13}
                                & & linear & RBF & poly 
                                & $I_\text{RBF}$ & $I_\text{poly}$ 
                                & GAK & VRK
                                & $S_\text{lin}$ & $S_\text{RBF}$ & $S^\infty_\text{RBF}$ & $S^\text{rand}_\text{tanh}$ \\ 
        \hline
        \hline
		\hline
		\multirow{2}{*}{EP}    
		& C & .89 & .95 & .91 & .94 & .91 & .97 & .94 & \textbf{.98} & \textbf{.98} & .92 & .94\\
		& M & .70 & .81 & .80 & .80 & .81 & .88 & .90 & \textbf{.98} & .97 & .91 & .95\\
		\hline
		\multirow{2}{*}{EC}    
		& C & .55 & .56 & .55 & .55 & .55 & .56 & .58 & \textbf{.59} & .55 & .56 & .55\\
		& M & .57 & .56 & .58 & .55 & .57 & .57 & \textbf{.60} & .56 & .56 & .56 & .56\\
		\hline
		\multirow{2}{*}{FM}    
		& C & .58 & .52 & .54 & .54 & .53 & \textbf{.60} & .53 & .49 & .49 & .48 & .54\\
		& M & \textbf{.58} & .51 & .53 & .50 & .52 & .54 & .48 & .49 & .48 & .51 & .56\\
		\hline
		\multirow{2}{*}{HMD}    
		& C & .55 & .43 & .54 & .53 & .49 & .46 & .52 & .53 & .48 & .50 & \textbf{.57}\\
		& M & .45 & .50 & .47 & .46 & .51 & \textbf{.55} & .54 & .50 & .50 & .52 & .52\\
		\hline
		\multirow{2}{*}{HB}    
		& C & .63 & .64 & .60 & .61 & .61 & .59 & .67 & .70 & \textbf{.72} & .62 & .61\\
		& M & .61 & .59 & .61 & .62 & .59 & .61 & .62 & \textbf{.69} & .60 & .59 & .60\\
		\hline
		\multirow{2}{*}{LSST}    
		& C & .54 & .61 & .53 & .61 & .56 & \textbf{.68} & .53 & .57 & .67 & .62 & .62\\
		& M & .62 & \textbf{.68} & .66 & .67 & .66 & .67 & .67 & .67 & .65 & .63 & .67\\
		\hline
		\multirow{2}{*}{MI}    
		& C & .51 & .54 & .57 & .57 & .53 & .50 & .54 & .43 & .57 & \textbf{.60} & .45\\
		& M & .51 & .52 & .47 & .46 & .49 & .50 & \textbf{.54} & .47 & .49 & .43 & .46\\
		\hline
		\multirow{2}{*}{PEMS}    
		& C & .91 & .92 & .90 & .91 & .89 & \textbf{.93} & .90 & \textbf{.93} & .92 & \textbf{.93} & .87\\
		& M & .48 & .69 & .53 & .66 & .52 & .77 & \textbf{.90} & .80 & .79 & .71 & .72\\
		\hline
		\multirow{2}{*}{PS}    
		& C & .62 & .65 & .65 & .64 & .63 & .66 & .67 & \textbf{.70} & .69 & .56 & .67\\
		& M & .65 & .67 & .65 & .65 & .64 & .65 & .68 & \textbf{.71} & .70 & .54 & .69\\
		\hline
		\multirow{2}{*}{RS}    
		& C & .79 & .73 & .74 & .80 & \textbf{.81} & .77 & .46 & .73 & .77 & .68 & .76\\
		& M & .34 & .58 & .48 & .60 & .42 & \textbf{.83} & .61 & .79 & .73 & .75 & .69\\
		\hline
		\multirow{2}{*}{SRS1}    
		& C & .68 & \textbf{.81} & .79 & .80 & \textbf{.81} & .77 & \textbf{.81} & .61 & .77 & .71 & .77\\
		& M & .73 & .60 & .70 & .59 & .58 & .62 & .72 & \textbf{.77} & \textbf{.77} & \textbf{.77} & .75\\
		\hline
		\multirow{2}{*}{SRS2}    
		& C & \textbf{.57} & .51 & .53 & .53 & .53 & .54 & .50 & .49 & .48 & .53 & .50\\
		& M & .57 & .55 & .54 & .55 & \textbf{.59} & .55 & .53 & .54 & .50 & .50 & .52\\
		\hline
		\hline
		\hline
		\multirow{2}{*}{Avg. AUC}    
		& C & .65 & .66 & .65 & \textbf{.67} & .65 & \textbf{.67} & .64 & .65 & \textbf{.67} & .64 & .66\\
		& M & .57 & .60 & .58 & .59 & .58 & .65 & .65 & \textbf{.66} & .65 & .62 & .64\\
		\hline
		\multirow{2}{*}{Avg. Rank}    
		& C & 11.0 & 10.7 & 10.9 & 10.0 & 11.8 & \textbf{8.8} & 9.9 & 10.4 & \textbf{8.8} & 12.4 & 11.1\\
		& M & 13.5 & 13.2 & 14.1 & 14.8 & 14.8 & 9.6 & 10.2 & \textbf{8.8} & 12.1 & 15.0 & 11.3\\
		\bottomrule
    \end{tabular}
\end{minipage}}
\caption{One-versus-rest ROC-AUC for the semi-supervised anomaly detection experiments on the UEA multivariate time series repository. The conformance and Mahalanobis methods are denoted by C and M, respectively. The symbols $I$, $S$, $S^\infty$ and $S^\text{rand}$ represent the integral, truncated signature, PDE signature, and randomized signature kernels, respectively.}
\label{tab:scores_roc_auc}
\end{table}

\begin{table}[]
\resizebox{0.90\linewidth}{!}{\begin{minipage}{\linewidth}
    \begin{tabular}{lc||ccc|cc|c|c|cccc}
        \toprule
        \multirow{2}{*}{Data set}   &  \multicolumn{12}{c}{Precision-Recall AUC} \\
        \cline{3-13}
                                & & linear & RBF & poly 
                                & $I_\text{RBF}$ & $I_\text{poly}$ 
                                & GAK & VRK
                                & $S_\text{lin}$ & $S_\text{RBF}$ & $S^\infty_\text{RBF}$ & $S^\text{rand}_\text{tanh}$ \\ 
        \hline
        \hline
		\hline
		\multirow{2}{*}{EP}    
		& C & .79 & .88 & .81 & .87 & .75 & .92 & .86 & \textbf{.96} & \textbf{.96} & .80 & .85\\
		& M & .42 & .58 & .54 & .57 & .58 & .73 & .73 & \textbf{.95} & .94 & .78 & .89\\
		\hline
		\multirow{2}{*}{EC}    
		& C & .28 & .28 & .29 & .28 & .31 & .29 & \textbf{.32} & \textbf{.32} & .29 & .30 & .28\\
		& M & .30 & .32 & .31 & .30 & .32 & .32 & \textbf{.33} & .28 & .28 & .30 & .30\\
		\hline
		\multirow{2}{*}{FM}    
		& C & .56 & .52 & .55 & .57 & .56 & \textbf{.60} & .54 & .54 & .52 & .53 & .54\\
		& M & \textbf{.55} & .51 & .52 & .50 & .52 & .52 & .51 & .52 & .49 & .53 & \textbf{.55}\\
		\hline
		\multirow{2}{*}{HMD}    
		& C & .30 & .24 & \textbf{.31} & .30 & .28 & .26 & .27 & .28 & .29 & .27 & .29\\
		& M & .25 & .29 & .26 & .28 & .27 & .31 & \textbf{.33} & .29 & .29 & \textbf{.33} & .27\\
		\hline
		\multirow{2}{*}{HB}    
		& C & .58 & .60 & .53 & .56 & .55 & .55 & \textbf{.63} & \textbf{.63} & \textbf{.63} & .58 & .60\\
		& M & .58 & .56 & .57 & .61 & .55 & .59 & .61 & \textbf{.64} & .55 & .56 & .58\\
		\hline
		\multirow{2}{*}{LSST}    
		& C & .12 & .12 & .10 & .15 & .11 & .14 & .10 & .12 & \textbf{.19} & .14 & .10\\
		& M & .14 & .15 & .16 & \textbf{.17} & .14 & .14 & .15 & \textbf{.17} & \textbf{.17} & .14 & .15\\
		\hline
		\multirow{2}{*}{MI}    
		& C & .54 & .54 & .57 & .57 & .54 & .49 & .55 & .47 & .56 & \textbf{.60} & .49\\
		& M & .55 & .54 & .53 & .49 & .53 & .53 & \textbf{.57} & .51 & .50 & .46 & .48\\
		\hline
		\multirow{2}{*}{PEMS}    
		& C & .79 & .82 & .79 & .81 & .79 & \textbf{.83} & \textbf{.83} & \textbf{.83} & .81 & .82 & .72\\
		& M & .34 & .41 & .33 & .39 & .31 & .48 & \textbf{.64} & .50 & .52 & .38 & .40\\
		\hline
		\multirow{2}{*}{PS}    
		& C & .05 & .06 & .06 & .05 & .05 & \textbf{.07} & .06 & \textbf{.07} & \textbf{.07} & .03 & .06\\
		& M & .05 & .06 & .05 & .05 & .05 & .05 & .06 & .07 & \textbf{.08} & .03 & .07\\
		\hline
		\multirow{2}{*}{RS}    
		& C & .65 & .64 & .66 & \textbf{.72} & .70 & .62 & .40 & .56 & .67 & .52 & .55\\
		& M & .20 & .33 & .26 & .34 & .23 & \textbf{.66} & .36 & .61 & .53 & .57 & .44\\
		\hline
		\multirow{2}{*}{SRS1}    
		& C & .69 & \textbf{.78} & .76 & \textbf{.78} & .77 & .74 & .76 & .63 & .76 & .70 & .75\\
		& M & \textbf{.77} & .67 & .69 & .66 & .64 & .69 & .68 & .75 & .73 & .74 & .74\\
		\hline
		\multirow{2}{*}{SRS2}    
		& C & \textbf{.57} & .51 & .53 & .54 & .52 & .55 & .52 & .51 & .50 & .54 & .51\\
		& M & .55 & .55 & .54 & .56 & \textbf{.58} & .55 & .54 & .55 & .51 & .51 & .52\\
		\hline
		\hline
		\hline
		\multirow{2}{*}{Avg. AUC}    
		& C & .49 & .50 & .50 & \textbf{.52} & .49 & .51 & .49 & .49 & \textbf{.52} & .49 & .48\\
		& M & .39 & .41 & .40 & .41 & .39 & .46 & .46 & \textbf{.49} & .47 & .44 & .45\\
		\hline
		\multirow{2}{*}{Avg. Rank}    
		& C & 11.0 & 11.1 & 10.6 & 8.0 & 11.4 & 10.3 & 9.8 & 10.6 & \textbf{7.7} & 11.2 & 12.8\\
		& M & 13.2 & 13.1 & 15.2 & 13.8 & 15.1 & 10.5 & 10.0 & \textbf{8.8} & 12.4 & 14.5 & 12.0\\
		\bottomrule
    \end{tabular}
\end{minipage}}
\caption{One-versus-rest precision-recall AUC for the semi-supervised anomaly detection experiments on the UEA multivariate time series repository. The conformance and Mahalanobis methods are denoted by C and M, respectively. The symbols $I$, $S$, $S^\infty$ and $S^\text{rand}$ represent the integral, truncated signature, PDE signature, and randomized signature kernels, respectively.}
\label{tab:scores_pr_auc}
\end{table}

\begin{figure}
  \begin{minipage}{\textwidth}
    \begin{tabular}{c@{\extracolsep{0pt}}c}
      \centering
      \includegraphics[width=0.5\textwidth]{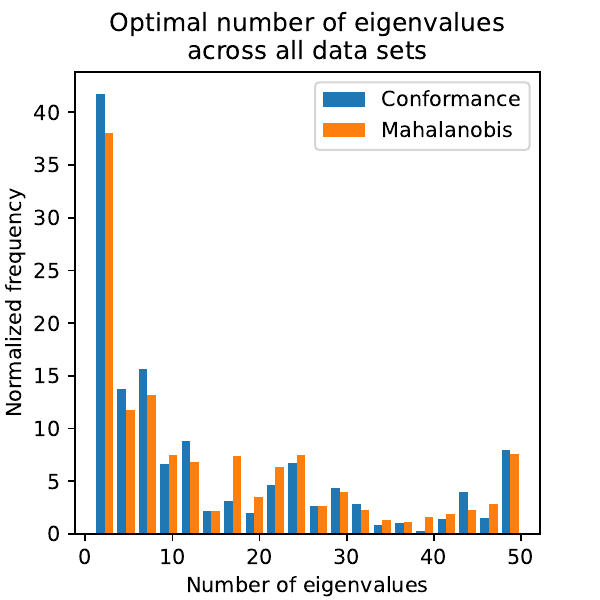} &
      \centering
      \includegraphics[width=0.5\textwidth]{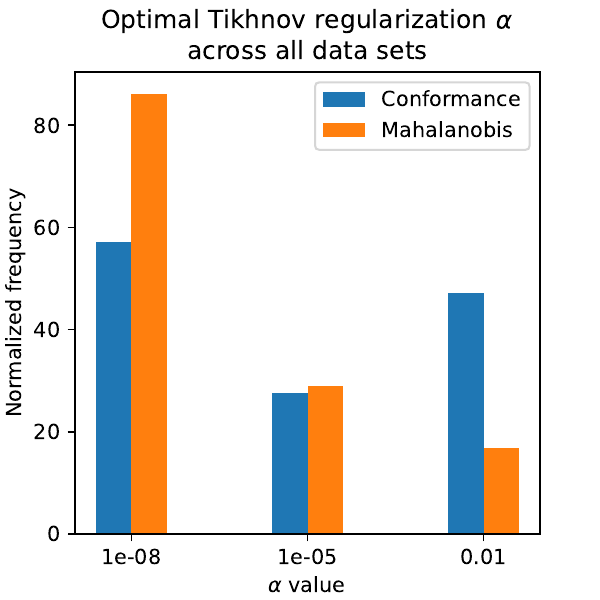}
    \end{tabular}
    \caption{Optimal hyper-parameters for computing the anomaly distance as per \crefrange{algoGramMatrix}{algoConformance}, sampled across all data sets and all kernels, normalized by the number of classes per data set. The results were obtained via a repeated $k$-fold cross-validation on the train set.}
    \label{fig:CV-alpha-threshold}
  \end{minipage}
\end{figure}

\begin{figure}
    \centering
    \includegraphics[width=0.59\textwidth]{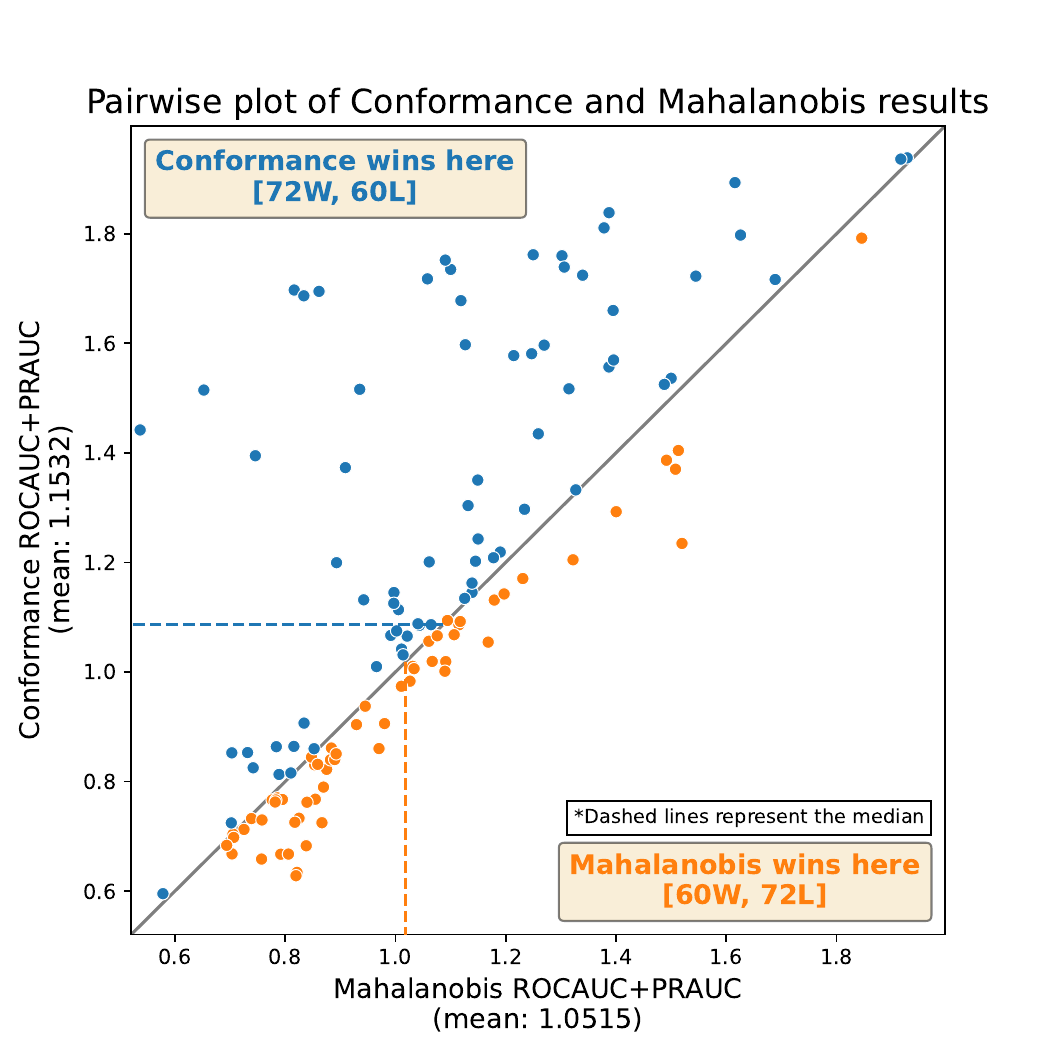}
    \caption{Pairwise comparison of one-versus-rest test scores for the Mahalanobis distance and the Conformance score. Each point represents one kernel and one data set.}
    \label{fig:pairwise-scatter}
\end{figure}

\subsubsection{Discussion}

For the Mahalanobis distance, there seems to be a clear advantage to working in the kernelized setting, as the results show that the linear $(L^2[0,1])^d$ inner product achieves the lowest average test scores out of all methods, with ROC-AUC and PR-AUC scores of 0.57 and 0.39, respectively. The GAK, VRK and truncated signature kernels on the other hand perform best in this regard, obtaining AUC scores of 0.65-0.66 and 0.46-0.49, respectively. 

The average test scores for the conformance score (nearest-neighbour Mahalanobis distance) do not differ much between the choices of kernels, but can have significant differences within a single data set. The average ROC and PR AUC scores calculated over all data sets range from 0.64-0.67 and 0.49-0.52, respectively, with the best results obtained from the RBF integral, RBF signature, and GAK kernel. 

When it comes to average rank, the Mahalanobis linear truncated signature and conformance RBF truncated signature take the number one spot, with the VRK and GAK kernels as close second place contenders. These kernels also have the most number of first places across all data sets, especially the linear truncated signature kernel. However, since the results are very data set dependent, the best performing model and kernel combination will vary on a case-by-case basis. 

\cref{fig:pairwise-scatter} shows a pairwise scatter plot of the Mahalanobis distance and conformance score test results for all kernels and all data sets. The results suggest that most of the time there is no significant advantage to using one anomaly distance over the other, except for a few cases seen in the upper left quadrant where the conformance score greatly outperforms the Mahalanobis distance. The difference in performance seem to be more pronounced for the simple flattened and integral-class kernels, where the average difference is 0.07 points, as opposed to the dynamic-time kernels where the average difference is 0.02 points. This difference is more pronounced for the PR-AUC metric, and two interesting examples are \textsc{RacketSports} and \textsc{PEMS-SF} where the PR-AUC doubles for select kernels when using the conformance method.

When it comes to computing the variance norm according to \crefrange{algoGramMatrix}{algoConformance}, both the Mahalanobis and conformance methods on average obtained their highest cross validation scores using a low number of eigenvalues, as seen in \cref{fig:CV-alpha-threshold}. Furthermore, we see that both methods in general preferred a low regularization parameter $\alpha$, with $\alpha=1\text{e-08}$ being most commonly used.


\acks{TC has been supported by the EPSRC Programme Grant EP/S026347/1 and acknowledges the support of the Erik Ellentuck Fellowship at the Institute for Advanced Study. NZ has been supported by the Roth Scholarship. We acknowledge computational resources and support provided by the Imperial College Research Computing Service (DOI: \texttt{10.14469/hpc/2232}). For the purpose of open access, the authors have applied a Creative Commons Attribution (CC BY) licence to any Author Accepted Manuscript version arising. We would like to thank the anonymous reviewers for their helpful comments on earlier versions of the manuscript which helped significantly improve the paper.}

\vskip 0.2in
\printbibliography

\end{document}